\documentclass[10pt,journal,compsoc]{IEEEtran}
\pdfoutput=1
\ifCLASSOPTIONcompsoc
  \usepackage[nocompress]{cite}
\else
  \usepackage{cite}
\fi

\usepackage[utf8]{inputenc}
\usepackage{subfig}
\usepackage{graphicx}
\usepackage{hyperref}
\usepackage{amsthm}
\usepackage{amsmath}
\usepackage{amssymb}
\usepackage{latexsym}
\usepackage{amsfonts}
\usepackage{amsmath}
\usepackage{graphicx}
\usepackage{color}
\usepackage{mathtools}
\usepackage{enumitem}
\usepackage[T1]{fontenc}
\usepackage[english]{babel}
\usepackage{float}
\usepackage{rotating}
\usepackage[table]{xcolor}
\usepackage{multirow}
\usepackage[ruled]{algorithm2e}
\usepackage {tikz}
\usepackage{thmtools,thm-restate}
\usepackage{footmisc}
\usepackage{chngcntr}
\usepackage{apptools}
\usepackage{caption}
\usepackage{stfloats}
\usepackage{lipsum,multicol}

\usepackage[normalem]{ulem}

\usetikzlibrary {positioning}
\usetikzlibrary{shapes.geometric, arrows}

\definecolor{darkGreen}{rgb}{0,0,0.99}
\definecolor{darkRed}{rgb}{0.99,0,0}
\definecolor {processblue}{cmyk}{0.96,0,0,0}

\tikzstyle{Lloyd0} = [rectangle, rounded corners, minimum width=3.5cm, minimum height=0.7cm,text centered,text width=3.5cm, draw=black, fill=green!30]
\tikzstyle{Lloyd2} = [rectangle, rounded corners, minimum width=2cm, minimum height=0.9cm,text centered,text width=2cm, draw=black, fill=green!30]
\tikzstyle{Lloyd} = [rectangle, rounded corners, minimum width=0.5cm, minimum height=0.7cm,text centered,text width=0.5cm, draw=white, fill=white!30]
\tikzstyle{Initz} = [rectangle, rounded corners, maximum width=7.8cm, minimum height=0.7cm,text centered, text width=7.8cm,draw=black, fill=orange!30]
\tikzstyle{Initz0} = [rectangle, rounded corners, maximum width=7.8cm, minimum height=0.7cm,text centered, text width=7.8cm,draw=black, fill=blue!30]

\tikzstyle{arrow} = [thick,->,>=stealth]

\AtAppendix{\counterwithin{lemma}{section}}
\AtAppendix{\counterwithin{thm}{section}}
\newtheorem{lemma}{Lemma}

\DeclareMathOperator*{\argmin}{arg\,min}

\newtheorem{dfn}{Definition}
\newtheorem{thm}{Theorem}

\definecolor{sarandonga}{rgb}{0.65,0.35,0}

\definecolor{BlueGreen}{rgb}{0.00,0.15,0.85}

\definecolor{Greeno}{rgb}{0.00,0.80,0.20}

\makeatletter
\newcommand{\removelatexerror}{\let\@latex@error\@gobble}
\makeatother

\makeatletter
\renewcommand\paragraph{\@startsection{paragraph}{4}{\z@}%
            {-2.5ex\@plus -1ex \@minus -.25ex}%
            {1.25ex \@plus .25ex}%
            {\normalfont\normalsize\bfseries}}
\makeatother
\ifCLASSINFOpdf
\else
\fi

\begin{document}

\title{An efficient $K$-means clustering algorithm for massive data}

\author{Marco~Capó,
        Aritz~Pérez,
        and~Jose~A.~Lozano
\IEEEcompsocitemizethanks{\IEEEcompsocthanksitem M. Capó and A. Pérez are at the Basque
Center of Applied Mathematics, Bilbao, Spain, 48009.\protect\\
E-mail: mcapo@bcamath.org , aperez@bcamath.org
\IEEEcompsocthanksitem J.A. Lozano is with the Intelligent Systems
Group, Department of Computer Science and Artifitial Intelligence, University of the Basque Country UPV/EHU, San Sebastián, Spain, 20018.\protect\\
E-mail: ja.lozano@ehu.es}}

\markboth{Journal of \LaTeX\ Class Files,~Vol.~14, No.~8, August~2015}%
{Shell \MakeLowercase{\textit{et al.}}: Bare Demo of IEEEtran.cls for Computer Society Journals}

\IEEEtitleabstractindextext{%
\begin{abstract}
The analysis of continously larger datasets is a task of major importance in a wide variety of scientific fields. In this sense, cluster analysis algorithms are a key element of exploratory data analysis, due to their easiness in the implementation and relatively low computational cost. Among these algorithms, the $K$-means algorithm stands out as the most popular approach, besides its high dependency on the initial conditions, as well as to the fact that it might not scale well on massive datasets. In this article, we propose a recursive and parallel approximation to the $K$-means algorithm  that scales well on both the number of instances and dimensionality of the problem, without affecting the quality of the approximation. In order to achieve this, instead of analyzing the entire dataset, we work on small weighted sets of points that mostly intend to extract information from those regions where it is harder to determine the correct cluster assignment of the original instances. In addition to different theoretical properties, which deduce the reasoning behind the algorithm, experimental results indicate that our method outperforms the state-of-the-art in terms of the trade-off between number of distance computations and the quality of the solution obtained.
\end{abstract}

}

\maketitle

\IEEEdisplaynontitleabstractindextext

%
\IEEEpeerreviewmaketitle

\IEEEraisesectionheading{\section{Introduction}\label{sec:introduction}}

\IEEEPARstart{P}{artitional} clustering is an unsupervised data analysis technique
that intends to unveil the inherent structure of a 
set of points by partitioning it into a number of disjoint groups, called 
clusters. This is done in such 
a way that intra-cluster similarity is high and the inter-cluster similarity
is low. Furthermore, clustering is a basic task in many areas, such as 
artificial intelligence, machine learning and pattern recognition 
\cite{Dubes,Jain,Kanungo2}.

Even when there exists a wide variety of clustering methods, the $K$-means algorithm remains as one of the most popular \cite{Berkhin, Jain2}. In fact, it has been identified as one of the top $10$ algorithms in data mining \cite{Wu}.

\subsection{$K$-means Problem}\label{sec:KP} 

Given a set of $n$ data points (instances) $D=\{\textbf{x}_1,\ldots,\textbf{x}_n\}$ in $\mathbb{R}^d$ and an integer $K$, the {\bf $K$-means problem} is to determine a set of $K$ centroids $C=\{\textbf{c}_1,\ldots,\textbf{c}_K\}$ in $\mathbb{R}^d$, so as to minimize the following error function:

\begin{eqnarray}\label{eq:errorfunction}
E^{D}(C) = \sum\limits_{\textbf{x}\in D} \| \textbf{x}-\textbf{c}_{\textbf{x}}\|^2, \text{ where  } \ \ \textbf{c}_{\textbf{x}}=\argmin\limits_{\textbf{c} \in C} \| \textbf{x}-\textbf{c}\|^2 
\end{eqnarray}

This is a combinatorial optimization problem, since it is equivalent to finding the partition
of the $n$ instances in $K$ groups whose associated set of centers of mass minimizes Eq.\ref{eq:errorfunction}.
The number of possible partitions is a Stirling number of the second kind,
$S(n,K)= \frac{1}{K!} \sum\limits_{j=0}^{K} (-1)^{K-j}\binom{K}{j}j^n$ \cite{Sami}.

Since finding the globally optimal partition is known to be NP-hard \cite{Aloise}, even 
for instances in the plane \cite{Mahajan}, and exhaustive search methods are not
useful under this setting, iterative refinement based algorithms are commonly used to
approximate the solution of the $K$-means and similar problems \cite{Sami,Kaufman, Lloyd}. 
These algorithms iteratively relocate the data points between clusters until a 
locally optimal partition is attained. Among these methods, the most popular is
the { \bf $K$-means algorithm} \cite{Jain2, Lloyd}.

\subsection{$K$-means Algorithm}\label{sec:KA} 

The $K$-means algorithm is an iterative refinement method that consists of 
two stages: Initialization, in which we set the starting set of $K$ centroids, and an iterative stage, called {\bf Lloyd's algorithm} \cite{Lloyd}. 
In the first step of Lloyd's algorithm, each instance is assigned to its closest centroid (assignment step), then the set of centroids is updated as the centers of mass of the instances assigned to the same centroid in the previous step (update step). Finally, a stopping criterion
is verified. The most common criterion implies the computation of the error function (Eq.\ref{eq:errorfunction}) : If the error does not change significantly with respect to the previous iteration, the algorithm stops \cite{Manning}: If $C$ and $C'$ are the set of centroids obtained at consecutive Lloyd's iterations, then the algorithm stops when

\begin{eqnarray}\label{eq:condpar}
|E^{D}(C)-E^{D}(C')|\leq \varepsilon, \text{ for a fixed threshold  } \ \ \varepsilon\ll 1. 
\end{eqnarray}

Conveniently, every step of the $K$-means algorithm can be easily parallelized \cite{Zhao}, which is a major key to meet the scalability of the algorithm \cite{Wu}.

The time needed for the assignment step is  $\mathcal{O}(n\cdot K\cdot d)$, while updating the set of centroids requires $\mathcal{O}(n\cdot d)$ computations and the stopping criterion, based on the computation of the error function, is $\mathcal{O}(n\cdot d)$. Hence, the assignment step is the most computationally demanding and this is due to  the number of distance computations that needs to be done at this step. Taking this into account, the main objective of our proposal is to define a variant of the $K$-means algorithm that controls the trade-off between the number of distance computations and the quality of the solution obtained, oriented to problems with high volumes of data. Lately, this problem has gained special attention due to the exponential increase of the data volumes that scientists, from different backgrounds, face on a daily basis, which hinders the analysis and characterization of such an
information \cite{Jordan}. 

\subsubsection{Common initializations}\label{sec:KI}

It is widely reported in the literature that the performance of Lloyd's algorithm highly depends upon the initialization stage, in terms of the quality of the solution obtained and the running time \cite{Lozano}. A poor initialization, for instance, could lead to an exponential running time in the worst case scenario \cite{Vattani}. 

Ideally, the selected seeding/initialization strategy should deal with different problems, such as outlier detection and cluster oversampling. A lot of research has been done on this topic: A detailed review of seeding strategies can be found in \cite{Redmond, Steinley}.

The standard initialization procedure consists of performing several re-initializations via Forgy's method \cite{Forgy} and keeping the set of centroids with the smallest error \cite{Redmond, Steinley}. Forgy's technique defines the initial set of centroids as $K$ instances selected uniformly at random from the dataset. The intuition behind this approach is that, by choosing the centroids uniformly at random, we are more likely to choose a point near an optimal cluster center, since such points tend to be where the highest density regions are located. Besides the fact that computing the error of each set of centroids is $\mathcal{O}(n\cdot K\cdot d)$ (due to the assignment step), the main disadvantage of this approach is that there is no guarantee that two, or more, of the selected seeds will not be near the center of the same cluster, especially when dealing with unbalanced clusters \cite{Redmond}.

More recently, simple probabilistic based seeding techniques have been developed and, due to their simplicity
and strong theoretical guarantees, they have become quite popular. Among these,
the most relevant is the {\it $K$-means++} algorithm
proposed by Arthur and Vassilvitskii in \cite{Arthur}. $K$-means$++$ selects only the first centroid uniformly at random from the dataset. Each subsequent initial centroid is 
chosen with a probability proportional to the distance with respect to the previously selected set of centroids.

The key idea of this cluster initialization technique is to preserve the diversity of seeds while being robust to outliers. The $K$-means$++$ algorithm leads to a $\mathcal{O}(\log K)$ factor approximation 
\footnote{Algorithm $A$ is an $\alpha$ factor approximation of the $K$-means problem, if 
$E^D(C') \leq \alpha \cdot \min\limits_{C \subseteq \mathbb{R}^d, |C|=K} E^{D}(C)$, for any output $C'$ of $A$.} of the optimal error after the 
initialization \cite{Arthur}. The main drawbacks of this approach refer to its sequential nature, which hinders its parallelization, as well as to the fact that it 
requires $K$ full scans of the entire 
dataset, which leads to a complexity of $\mathcal{O}(n\cdot K\cdot d)$.

In order to alleviate such drawbacks, different variants of $K$-means$++$ have been studied.
In particular, in \cite{Bahmani}, a parallel $K$-means$++$ type algorithm is presented. This parallel variant
achieves a constant factor approximation to the optimal solution after a logarithmic number
of passes over the dataset. Furthermore, in \cite{Bachem},
an approximation to $K$-means$++$ that has a sublinear 
time complexity, with respect to the number of data points,
is proposed. Such an approximation is obtained
via a Markov Chain Monte Carlo sampling based approximation of the 
$K$-means$++$ probability function. The proposed algorithm generates solutions of similar quality to those of $K$-means$++$, at a fraction
of its cost.

\subsubsection{Alternatives to Lloyd's algorithm}\label{sec:MBK} 

Regardless of the initialization, a large amount of work has also been done on reducing the overall computational complexity 
of Lloyd's algorithm. Mainly, two approaches can be distinguished:

\begin{itemize}[leftmargin=*]
 \item {\bf The use of distance pruning techniques}: Lloyd's algorithm can be 
accelerated by avoiding unnecessary distance calculations, i.e., when it can be verified in advanced that no cluster re-assignment is
possible for a certain instance. As presented in \cite{Drake,Elkan,Hamerly}, this can be done with the construction of different
pairwise distance bounds between the set of points and centroids, and additional information, such as
the displacement of every centroid after a Lloyd's iteration.
In particular, in \cite{Hamerly}, reductions of over 80\% of the amount of distance computations are observed.\\

\item {\bf Apply Lloyd's algorithm over a smaller (weighted) set of points}: As previously commented, one of the main
drawbacks of Lloyd's algorithm is that its complexity is proportional to the size of the dataset, meaning that
it may not scale well for massive data 
applications. One way of dealing with this is to apply the algorithm over a smaller set of points rather
than over the entire dataset. Such smaller sets of points are commonly extracted in two different ways:

\begin{enumerate}[wide, labelwidth=!, labelindent=0pt]
 \item {\it Via dataset sampling}: In \cite{Bengio, Bradley, Davidson, Sculley}, 
different statistical techniques are used with the same purpose of reducing the size of the dataset. Among these algorithms, we have the {\it Mini-batch $K$-means} proposed by Sculley in \cite{Sculley}. 
Mini-batch $K$-means is a very popular scalable variant of Lloyd's algorithm that 
proceeds as follows: Given an initial set of centroids obtained via Forgy's algorithm, at every iteration, 
a small fixed amount of samples is
selected uniformly at random and assigned to their corresponding cluster. Afterwards, the cluster centroids are updated as the 
average
of all samples ever assigned to them. This process continues until convergence.
Empirical results, in a range of large web based applications, 
corroborate that a substantial saving of computational time can be obtained at the expense of some loss of cluster quality \cite{Sculley}.
Moreover, very recently, in \cite{Newling}, an accelerated Mini-batch $K$-means algorithm
via the distance pruning approach of \cite{Elkan} was presented.

 \item {\it Via dataset partition}: The reduction of the dataset can also be generated as sets of representatives 
 induced by partitions of the dataset. In particular,
 there have been a number of recent papers that describe $(1 + \varepsilon)$ factor approximation 
algorithms and/or ($K$,$\varepsilon$)-coresets \footnote{
A weighted set of points $W$ is a ($K$,$\varepsilon$)-coreset if, for all set of centroids $C$,
$|F^{W}(C)-E^{D}(C)|\leq \varepsilon \cdot E^{D}(C)$, where $F^{W}(C)=\sum\limits_{\textbf{y}\in W} w(\textbf{y})\cdot \| \textbf{y}-\textbf{c}_{\textbf{y}}\|^2$ and $w(\textbf{y})$ is the weight associated to a representative $\textbf{y} \in W$.} for the
$K$-means problem  \cite{Har,Kumar,Matousek}. However, these variants tend to be exponential in $K$ and are not at all viable in practice \cite{Arthur}. Moreover,
Kanungo et al. \cite{Kanungo} also proposed a $(9+\varepsilon)$ approximated algorithm for the $K$-means problem that is $\mathcal{O}(n^3 \varepsilon^{-d})$, thus it is not useful for massive data applications. In particular, for this kind of applications, another approach has been very recently proposed in \cite{Capo}: The {\it Recursive Partition based $K$-means algorithm}.

\paragraph{ Recursive Partition based $K$-means algorithm}\label{sss:RPKM} 

The Recursive Partition based $K$-means algorithm ({\bf RP$K$M}) is a technique that
approximates the solution 
of the $K$-means problem through a recursive application of a weighted version of Lloyd's algorithm
over a sequence of spatial based-thinner partitions of the dataset: 

\begin{dfn}[Dataset partition induced by a spatial partition]
Given a dataset $D$ and a spatial partition $\mathcal{B}$ of its smallest bounding
box, the partition of the dataset $D$ induced by 
$\mathcal{B}$ is defined as $\mathcal{P}=\mathcal{B}(D)$, where $\mathcal{B}(D)=\{B(D)\}_{B \in \mathcal{B}}$ 
and $B(D)=\{\ \textbf{x}\in D: \textbf{x} \ \textrm{lies on}\ B \in \mathcal{B} \}$ \footnote{From now on, we will refer to each $B \in \mathcal{B}$ as a {\bf block} of the spatial partition $\mathcal{B}$.}.
\end{dfn}

Applying a weighted version of $K$-means algorithm over the dataset partition 
$\mathcal{P}$, consists of executing  Lloyd's
algorithm (Section \ref{sec:KA}) over the set of centers of mass ({\bf representatives}) of $\mathcal{P}$, $\overline{P}$ for all $P \in \mathcal{P}$,
considering their corresponding cardinality ({\bf weight}), $|P|$, when updating the set of centroids. This means, we seek to minimize the weighted error function $E^{\mathcal{P}}(C) = \sum\limits_{P \in \mathcal{P}} |P| \cdot \| \overline{P}-\textbf{c}_{\overline{P}}\|^2$, where $\textbf{c}_{\overline{P}}=\argmin\limits_{\textbf{c} \in C} \|\overline{P}-\textbf{c}\|$.

Afterwards, the same process is repeated over a thinner partition $\mathcal{P}'$ of the dataset
 \footnote{ A partition of the dataset $\mathcal{P}'$ is thinner than $\mathcal{P}$, if
each subset of $\mathcal{P}$ 
can be written as the union of subsets of $\mathcal{P}'$.}, using as initialization the set of centroids obtained for $\mathcal{P}$.
In Algorithm \ref{alg:RPKM_ORIGINAL}, we present a pseudo-code of a RP$K$M type algorithm:

\removelatexerror
\begin{algorithm}[H]\label{alg:RPKM_ORIGINAL}
    \caption{{\bf RP$K$M algorithm pseudo-code}}

   {\bf Input:} Dataset $D$ and number of clusters $K$.\\ 

   {\bf Output:} Set of centroids $C$.\\
   \vspace{0.1cm}
\texttt{Step 1:} Construct an initial partition of $D$, $\mathcal{P}$, and define an initial set of $K$ centroids, $C$.\\
   \texttt{Step 2:} $C =$ \texttt{WeightedLloyd}$(\mathcal{P}, C,K)$. \\
   \While{not Stopping Criterion}{
\texttt{Step 3:}  Construct a dataset partition $\mathcal{P}'$, thinner than $\mathcal{P}$. Set $\mathcal{P} = \mathcal{P}'$.

   \texttt{Step 4:} $C =$ \texttt{WeightedLloyd}$(\mathcal{P}, C,K)$. \\
}
\Return $C$

\end{algorithm}

In general, the RP$K$M algorithm can be divided into three tasks: The construction of an initial partition of the dataset and set of centroids (\texttt{Step 1}), the update of the corresponding set of centroids via weighted Lloyd's algorithm (\texttt{Step 2} and \texttt{Step 4}) and the construction of the sequence of thinner partitions (\texttt{Step 3}). Experimental results have shown the reduction of several orders of computations for RP$K$M with respect to both $K$-means++ and Mini-batch $K$-means, while obtaining competitive approximations to the solution
of the $K$-means problem \cite{Capo}.

\end{enumerate}

\end{itemize}

\subsection{Motivation and contribution}\label{sec:CTB}

In spite of the quality
of the practical results presented in \cite{Capo}, due to the strategy followed
in the construction of the sequence of thinner partitions, there is still large room
for improvement. The results presented in \cite{Capo} refer to 
a RP$K$M variant called {\bf grid based RP$K$M}.
In the case of the grid based RP$K$M, the initial spatial partition is 
defined by the grid obtained after dividing each side  of the smallest bounding box of $D$ by half, i.e., a grid with $2^d$ equally sized
blocks.
In the same fashion, at the $i$-th grid based RP$K$M iteration, the corresponding
spatial partition is
updated by dividing each of its
blocks into $2^{d}$ new blocks, i.e., $\mathcal{P}$ can have up to $2^{i \cdot d}$ representatives. It can be shown that this approach produces a ($K$,$\varepsilon$)-coreset with $\varepsilon$ descending exponentially with respect to the number of iterations \footnote{See Theorem \ref{theo:coreset} at Appendix \ref{App:AppendixA}.}. 

Taking this into consideration,
three main problems arise for the grid based RP$K$M:

\begin{itemize}[leftmargin=*]
 \item {\bf Problem 1.} {\it It does not scale well on the dimension $d$}: 
 Observe that, for a relatively low number of iterations, $i\simeq \log_{2}(n)/d$, 
 and/or dimensionality
 $d\simeq \log_{2}(n)$, applying this RP$K$M version can be similar to applying
 Lloyd's algorithm over the entire dataset, i.e., no reduction of distance computations might be observed,
 as $|\mathcal{P}|\simeq n$. 
 In fact, for the experimental section in \cite{Capo},  $d,i \leq 10$.
 \item {\bf Problem 2.} {\it It is independent of the dataset $D$}: As we noticed before, regardless of the analyzed dataset $D$, the sequence of partitions of the grid based RP$K$M is induced by an equally sized spatial partition of the smallest bounding box containing $D$. In this sense, the induced partition does not consider features of the dataset, such as its density, to construct the sequence of partitions: A large amount of computational resources might be spent on regions whose misclassification does not add a significant error to our approximation. Moreover, the construction of every partition of the sequence has a $O(n \cdot d)$ cost, which is particularly expensive for massive data applications, as $n$ can be huge.
 \item {\bf Problem 3.} {\it It is independent of the problem}: The partition strategy of the grid based RP$K$M does not explicitly consider the optimization problem that $K$-means seeks to minimize. Instead, it offers a simple/inefficient way of generating a sequence of spatial thinner partitions.

 $\ \ $ The reader should note that each block of the spatial partition can be seen as a restriction over the $K$-means optimization problem, that enforces all the instances contained in it to belong to the same cluster. Therefore, it is of our interest to design smarter spatial partitions oriented to focus most of the computational resources on those regions where the correct cluster affiliation is not clear. By doing this, not only can a large amount of computational resources be saved, but also some additional theoretical properties can be deduced.
 
 $\ \ $ Among other properties that we discuss in Section \ref{Sec:Contribution}, at first glance it can be observed that if all the instances in a set of points, $P$, are correctly assigned for two sets of centroids, $C$ and $C'$, then the difference between the error of both sets of centroids is equivalent to the difference of their weighted error, i.e., $E^{P}(C)-E^{P}(C')= E^{\{ P \}}(C)-E^{\{ P \}}(C')$ \footnote{\label{note1} See Lemma \ref{theo:wellassigned} in Appendix \ref{App:AppendixA}.}.
Moreover, if this occurs for each subset of a dataset partition $\mathcal{P}$ and the centroids generated after two consecutive weighted $K$-means iterations, then we can guarantee a
monotone decrease of the error for the entire dataset \footnote{\label{note1} See Theorem \ref{lemma:wellassigned} in Appendix \ref{App:AppendixA}.}.  Likewise, we
can actually compute the reduction of the error for the newly obtained set of centroids, without computing
the error function for the entire dataset, as in this case $E^{D}(C)-E^{D}(C')= E^{\mathcal{P}}(C)-E^{\mathcal{P}}(C')$. Last but not least, when every block contains instances belonging to the same cluster, the solution obtained
by our weighted approximation is actually a local optima of Eq.\ref{eq:errorfunction}
\footnote{\label{note1} See Theorem \ref{lemma:wellassigned2} in Appendix \ref{App:AppendixA}.}.
\end{itemize}

In any case, independently of the partition strategy, RP$K$M algorithm offers some interesting properties
such as the no clustering repetition. This is, none of the obtained groupings of the $n$ instances
into $K$ groups can be repeated at the current RP$K$M iteration or for any thinner partition than the current
one. This is a useful property since it can be guaranteed that the algorithm 
discards many possible clusterings at each RP$K$M iteration
using a much reduced set of points than the entire dataset. Furthermore, this fact
enforces the decrease of the maximum number
of Lloyd iterations that we can have for a given partition. In practice, it is also common to observe 
a monotone decrease of the error for the entire dataset \cite{Capo}.

Bearing all these facts in mind, we propose a RP$K$M type approach called the {\it Boundary Weighted $K$-means} algorithm ({\bf BW$K$M}). The name of our proposal summarizes the main intuition behind it: To generate competitive approximations to the $K$-means problem by dividing those blocks that may not be well assigned, which conform the current cluster boundaries of our weighted approximation.

\begin{dfn}[Well assigned blocks]\label{def:wellAssigned}
Let $C$ be a set of centroids and $D$ be a given dataset. We say that a block $B$ is well assigned with respect to $C$ and $D$ if every point $\textbf{x} \in B(D)$ is assigned to the same centroid $c \in C$. 
\end{dfn}
The notion of well assigned blocks is of our interest as RP$K$M associates all the instances contained in a certain block to the same cluster, which corresponds to the one that its center of mass belongs to. Hence, our goal is to divide those blocks that are not well assigned. Moreover, in order to control the growth of the set of representatives and to avoid unnecessary distance computations, we have developed a non-expensive partition criterion that allows us to detect blocks that may not be well assigned. Our main proposal can be divided into three tasks:

\begin{itemize}[leftmargin=*]
 \item {\bf Task 1}: Design of a partition criterion that decides whether or not to divide a certain block, using only information obtained from the  weighted Lloyd's algorithm.
 \item {\bf Task 2}: Construct an initial partition of the dataset given a fixed number of blocks, which are mostly placed on the cluster boundaries.
 \item {\bf Task 3}: Once a certain block is decided to be cut, guarantee a low increase on the number of representatives without affecting, if possible, the quality of the approximation. In particular, we propose a criterion that, in the worst case, has a linear growth in the number of representatives after an
iteration.
\end{itemize}

Observe that, both {\bf Task 2} and {\bf Task 3}, ease the scalability of the algorithm with respect to the dimensionality of the problem, $d$ ({\bf Problem 1}). Furthermore, the goal of {\bf Task 1} and {\bf Task 2} is to generate partitions of the dataset that, ideally, contain well assigned subsets, i.e., all the instances contained in a certain subset of the partition belong to the same cluster ({\bf Problem 2} and {\bf Problem 3}).
As we previously commented, this fact implies additional theoretical properties in terms of the quality of our approximation.

The rest of this article is organized as follows: In Section \ref{Sec:Contribution}, we describe the proposed algorithm, introduce some notation
and discuss some theoretical properties of our proposal. In Section \ref{Sec:Experimental}, we present a set of experiments in which we analyze the effect of different factors, such as the size of the dataset and the dimension of the instances over the performance of our algorithm. Additionally we compare these results with the ones obtained by the state-of-the-art. Finally, in Section \ref{Sec:NextSteps}, we define the next steps and possible improvements to our current work.

\section{BW$K$M algorithm} \label{Sec:Contribution}

In this section, we present the Boundary Weighted $K$-means algorithm. As we already commented, BW$K$M is a scalable improvement of the grid based RP$K$M algorithm \footnote{From now on, we assume each block $B \in \mathcal{B}$ to be a hyperrectangle.}, that generates competitive approximations to the $K$-means problem, while reducing the amount of computations that the state-of-the-art algorithms require for the same task. BW$K$M reuses all the information generated at each weighted Lloyd run to construct a sequence of thinner partitions that alleviates {\bf Problem 1}, {\bf Problem 2} and {\bf Problem 3}.

Our new approach makes major changes in all the steps in Algorithm \ref{alg:RPKM_ORIGINAL} except in \texttt{Step 2} and \texttt{Step 4}. In these steps, a weighted version of Lloyd's algorithm is applied over the set of representatives and weights of the current dataset partition $\mathcal{P}$. This process has a $O(|\mathcal{P}|\cdot K \cdot d)$ cost, hence it is of our interest to control the growth of $|\mathcal{P}|$, which is highlighted in both
{\bf Task 2} and {\bf Task 3}. 

In the following sections, we will describe in detail each step of BW$K$M. In Section \ref{SubSec:DetectBound}, Section \ref{SubSec:InitialPartition} and Section \ref{SubSec:ThinnerP} we elaborate on {\bf Task 1}, {\bf Task 2} and {\bf Task 3}, respectively.

\subsection{A cheap criterion for detecting well assigned blocks} \label{SubSec:DetectBound}

BW$K$M tries to efficiently determine the set of well assigned blocks in order to update the dataset partition. In the following definition, we introduce a criterion that will help us verify this mostly using information generated by our weighted approximation:

\begin{dfn}\label{def:epsilon}
Given  a set of $K$ centroids, $C$, a set of points $D \subseteq \mathbb{R}^d$, a block $B$ and $P=B(D)\ne \emptyset$ the subset of points contained in $B$. We define the misassignment function for $B$ given $C$ and $P$ as: 
\begin{eqnarray}\label{eq:criterio}
\epsilon_{C,D}(B)=\max\{0,2\cdot l_{B}-\delta_{P}(C)\},
\end{eqnarray}
where $\delta_{P}(C)=\min\limits_{\textbf{c} \in C \setminus \textbf{c}_{\overline{P}}} \| \overline{P}-\textbf{c}\|-\| \overline{P}-\textbf{c}_{\overline{P}}\|$ and $l_{B}$ is the length of the diagonal of $B$. In the case $P=B(D)=\emptyset$, we set $\epsilon_{C,D}(B)=0$. 
\end{dfn}

The following result is used in the construction of both the initial and the sequence of thinner partitions:

\begin{restatable}{thm}{primetheo}\label{thm:FronteraGratis1}
Given a set of $K$ centroids, $C$, a dataset, $D \subseteq \mathbb{R}^d$, and a block $B$, if  $\epsilon_{C,D}(B) = 0$, then $\textbf{c}_{\textbf{x}}= \textbf{c}_{\overline{P}}$ for all $\textbf{x} \in P=B(D)\ne \emptyset$.\footnote{The proof of Theorem \ref{thm:FronteraGratis1} is in Appendix \ref{App:AppendixA}.}
\end{restatable}

In other words, if the misassignment function of a block is zero, then the block is well assigned. Otherwise, the block may not be well assigned. Even though the condition in Theorem \ref{thm:FronteraGratis1} is a sufficient condition, we will use the following heuristic rule during the development of the algorithm: The larger the misassignment function of a certain block is, then the more likely it is to contain instances with different cluster memberships.

In particular, Theorem \ref{thm:FronteraGratis1} offers an efficient and effective way of verifying that all the instances contained in a block $B$ belong to the same cluster, using only information related to the structure of $B$ and the set of centroids, $C$. Observe that we do not need any information associated to the individual instances in the dataset, $\textbf{x} \in P$. The criterion just requires some distance computations with respect to the representative of $P$, $\overline{P}$, that are already obtained from the weighted Lloyd's algorithm. 

\begin{dfn}\label{def:boundary}
Let $D$ be a dataset, $C$ be a set of $K$ centroids and $\mathcal{B}$ be a spatial partition. We define the boundary of $\mathcal{B}$ given $C$ and $D$ as

\begin{equation}
\mathcal{F}_{C,D}(\mathcal{B})=\{B \in \mathcal{B}: \epsilon_{C,D}(B)> 0\}
\end{equation}
\end{dfn}

The boundary of a spatial partition is just the subset of blocks with a positive misassignment function value, that is, the blocks that may not be well assigned. In order to control the size of the spatial partition and the number of distance computations, BW$K$M only splits blocks from the boundary.

In Fig.\ref{fig:TheoExp}, we observe the information needed for a certain block of the spatial partition, the one marked out in black, to verify the criterion presented in Theorem \ref{thm:FronteraGratis1}. In this example, we only set two cluster centroids (blue stars) and the representative of the instances in the block, $\overline{P}$, given by the purple diamond. In order to compute the misassignment function of the block, we require the length of the {\bf three} segments: Distance between the representative with respect to its two closest centroids in $C$ (blue dotted lines) and the diagonal of the block (purple dotted line). If the misassignment function is zero, then we know that all the instances contained in the block belong to the same cluster. Observe that, in this example, there are instances in both red and blue clusters, the misassignment function is positive, thus, the block is included in the boundary.

\begin{figure}[H]
\begin{center}
    \includegraphics[height=0.28\textwidth,width=0.50\textwidth]{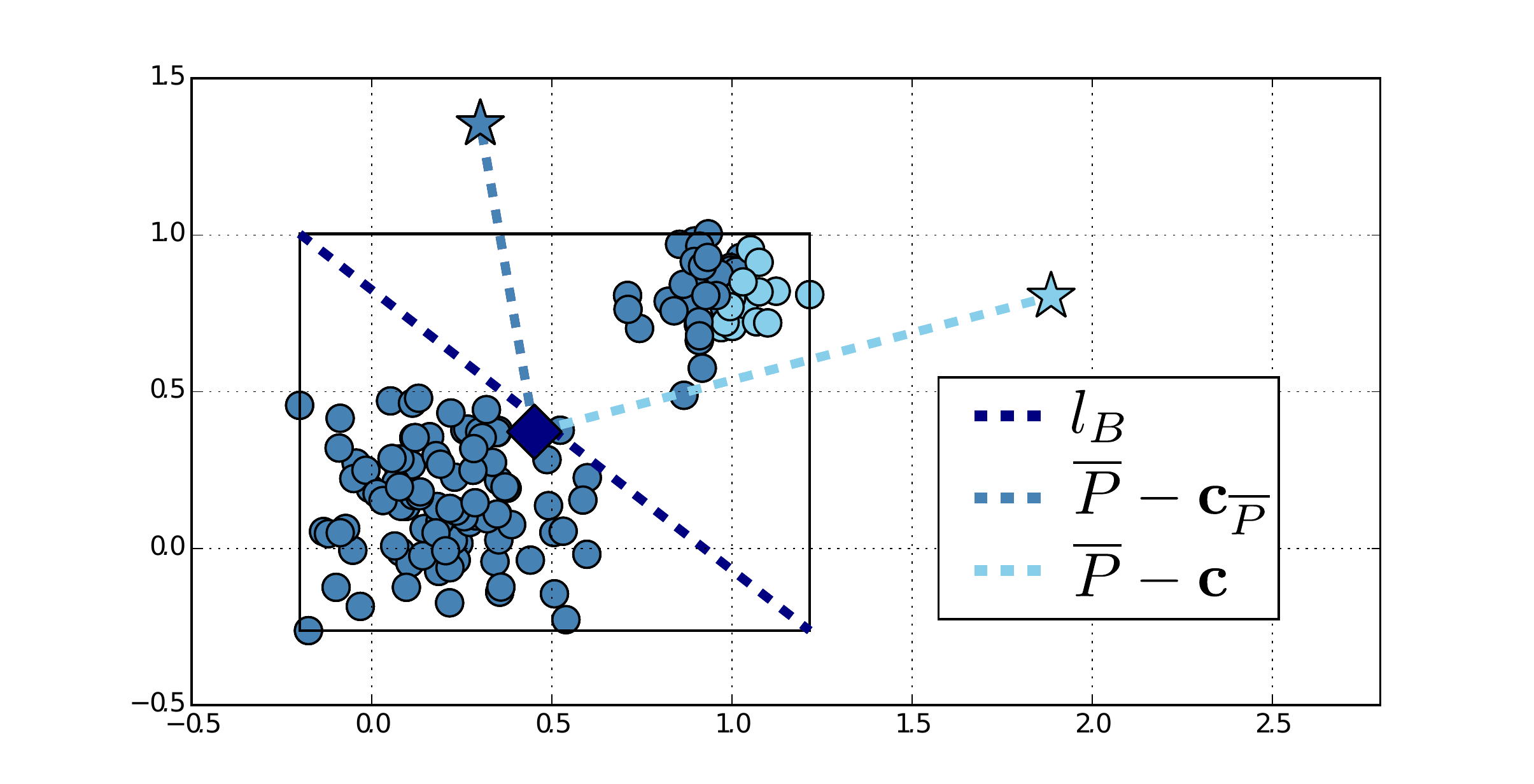}   
    \caption{Information required for computing the misassignment function of the block $B$,  $\epsilon_{C,D}(B)$, for $K=2$.}
    \label{fig:TheoExp}
\end{center}
\end{figure}

\begin{restatable}{thm}{sectheo}\label{thm:ErrBound2}
Given a dataset, $D$, a set of $K$ centroids $C$ and a spatial partition $\mathcal{B}$ of the dataset $D$, the following inequality is satisfied:

\begin{eqnarray}
|\hspace{-7mm}&&E^{D}(C)-E^{\mathcal{P}}(C)| \leq \nonumber\\ 
&&\sum\limits_{\substack{B \in \mathcal{B}}} 2 \cdot |P| \cdot \epsilon_{C,D}(B) \cdot (2 \cdot l_{B}+ \| \overline{P}- \textbf{c}_{\overline{P}} \|) + \frac{|P|-1}{2} \cdot l_{B}^2, \nonumber
\end{eqnarray}
where $P=B(D)$ and $\mathcal{P}= \mathcal{B}(D)$ \footnote{The proof of Theorem \ref{thm:ErrBound2} is in Appendix \ref{App:AppendixA}.}.
\end{restatable}

According to this result, we must increase the amount of well assigned blocks and/or reduce the diagonal lengths of the blocks of the spatial partition, so that our weighted error function  approximates better the $K$-means error function, Eq.\ref{eq:errorfunction}. Observe that by reducing the diagonal of the blocks, not only is the condition of Theorem \ref{thm:FronteraGratis1} more likely to be satisfied, but also we are directly reducing both additive terms of the bound in Theorem \ref{thm:ErrBound2}. This last point gives the intuition for our new partition strategy: i) split only those blocks in the boundary and ii) split them on their largest side.

\subsection{Initial Partition} \label{SubSec:InitialPartition}

In this section, we elaborate on the construction of the initial dataset partition used by the BW$K$M algorithm (see \texttt{Step 1} of Algorithm \ref{alg:RPKM_A}, where the main pseudo-code of BW$K$M is). Starting with the smallest bounding box of the dataset, the proposed procedure iteratively divides subsets of blocks of the spatial partition with high probabilities of not being well assigned. In order to determine these blocks, in this section we develop a probabilistic heuristic based on the misassignment function, Eq.\ref{eq:criterio}. 

As our new cutting criterion is mostly based on the evaluation of the misassignment function associated to a certain block, we firstly need to construct a starting spatial partition of size $m' \geq K$, from where we can select the set of $K$ centroids with respect to which the function is computed (\texttt{Step 1}). 

From then on, multiple sets of centroids $C$ are selected via a weighted $K$-means++ run over the set of representatives of the dataset partition, for different subsamplings. This will allow us to estimate a probability distribution that quantifies the chances of each block of not being well assigned (\texttt{Step 2}). Then, according to this distribution, we randomly select the most promising blocks to be cut (\texttt{Step 3}), and divide them until reaching a number of blocks $m$ (\texttt{Step 4}). In Algorithm \ref{alg:InitialPartition}, we show the pseudo-code of the algorithm proposed for generating the initial spatial partition.

  \removelatexerror

\begin{algorithm}[H]\label{alg:InitialPartition}
    \caption{{\bf Construction of the initial partition}}
   {\bf Input:} Dataset $D$, number of clusters $K$, integer $m'>K$, size of the initial spatial partition $m>m'$.\\
   {\bf Output:} Initial spatial partition $\mathcal{B}$ and its induced dataset partition, $\mathcal{P}=\mathcal{B}(D)$.\\
   \vspace{0.1cm}   
   \texttt{Step 1}: Obtain a starting spatial partition of size $m'$, $\mathcal{B}$ (Algorithm \ref{alg:StartingPartition}).\\
   \While{$|\mathcal{B}|< m$}{
   		\texttt{Step 2}: Compute the cutting probability, $Pr(B)$ for $B \in \mathcal{B}$ (Algorithm \ref{alg:computeProb}).\\
        \texttt{Step 3}: Sample $\min\{|\mathcal{B}|,m-|\mathcal{B}|\}$ blocks from $\mathcal{B}$, with replacement, according to $\Pr(\cdot)$ to determine a subset of blocks $\mathcal{A} \subseteq \mathcal{B}$. \\
   		\texttt{Step 4}: Split each $B \in \mathcal{A}$ and update $\mathcal{B}$.
   }
   \texttt{Step 5}: Construct $\mathcal{P}=\mathcal{B}(D)$.\\
\Return $\mathcal{B}$ and $\mathcal{P}$.
\end{algorithm}

In \texttt{Step 1}, a partition of the smallest bounding box containing the dataset $D$, $B_D$, of size $m'>K$ is obtained by splitting recursively the blocks according to the pseudo-code shown in Algorithm \ref{alg:StartingPartition} --see the comments below. Once we have the spatial partition of size $m'$, we iteratively produce thinner partitions of the space as long as the number of blocks is lower than $m$. At each iteration, the process is divided into three steps: In \texttt{Step 2}, we estimate the cutting probability $Pr(B)$ for each block $B$ in the current space partition $\mathcal{B}$ using Algorithm \ref{alg:computeProb} --see the comments below. Then, in \texttt{Step 3}, we 
randomly sample (with replacement) $\min\{|\mathcal{B}|,m-|\mathcal{B}|\}$ blocks from $\mathcal{B}$ according to $Pr(\cdot)$ to construct the subset of blocks $\mathcal{A} \subseteq \mathcal{B}$, i.e., $|\mathcal{A}|\leq \min\{|\mathcal{B}|,m-|\mathcal{B}|\} $. Afterwards, each of the selected blocks in $\mathcal{A}$ is replaced by two smaller blocks obtained by splitting $B$ in the middle point of its longest side. Finally, the obtained spatial partition $\mathcal{B}$ and the induced dataset partition $\mathcal{B}(D)$ (of size lower or equal to $m$) are returned. 

\begin{algorithm}[H]\label{alg:StartingPartition}
    \caption{{\bf Step 1 of Algorithm \ref{alg:InitialPartition}}}
   {\bf Input:} Dataset $D$, partition size $m'>K$, sample size $s<n$.\\
   {\bf Output:} A spatial partition of size $m'$, $\mathcal{B}$.\\
   \vspace{0.1cm}
   
   - Set $\mathcal{B}=\{B_D\}$.
   
   \While{$|\mathcal{B}|< m'$}{
   		- Take a random sampling of size $s$, $S \subset D$ .\\ 
        - Obtain a subset of blocks, $\mathcal{A} \subseteq \mathcal{B}$, by sampling, with replacement, $\min\{|\mathcal{B}|,m'-|\mathcal{B}|\}$ blocks according to a probability proportional to $l_B \cdot |B(S)|$, for each $B \in \mathcal{B}$.\\
   		- Split the selected blocks $\mathcal{A}$ and update $\mathcal{B}$.  
        }
\Return $\mathcal{B}$.
\end{algorithm}

Algorithm \ref{alg:StartingPartition} generates the starting spatial partition of size $m'$ of the dataset $D$. This procedure recursively obtains thinner partitions by splitting a subset of up to $\min\{|\mathcal{B}|,m'-|\mathcal{B}|\}$ blocks selected by a random sampling with replacement according to a probability proportional to the product of the diagonal of the block $B$, $l_B$, by its weight, $|B(S)|$. At this step, as we can not estimate how likely it is for a given block to be well assigned with respect to a set of $K$ representatives, the goal is to control both weight and size of the generated spatial partition, i.e., to reduce the possible number of cluster misassignments, as this cutting procedure prioritizes those blocks that might be large and dense.
Ultimately, as we reduce this factor, we improve the accuracy of our weighted approximation- see Theorem \ref{thm:ErrBound2}.

This process is repeated until a spatial partition with the desired number of blocks, $m'\geq K$, is obtained. Such a partition is later used to determine the sets of centroids which we use to verify how likely it is for a certain block to be well assigned. 

\begin{algorithm}[H]\label{alg:computeProb}
    \caption{{\bf Step 2 of Algorithm \ref{alg:InitialPartition}}}
   {\bf Input:} A spatial partition $\mathcal{B}$ of size higher than $K$, dataset $D$, number of clusters $K$, sample size $s$, number of repetitions $r$.\\
   {\bf Output:} Cutting probability $\Pr(B)$ for each $B \in \mathcal{B}$.\\
   \vspace{0.1cm}   
   \For{$i=1,\ldots,r$}{

	-Take subsample $S^i\subseteq D$ of size $s$ and construct $\mathcal{P}= \mathcal{B}(S^i)$.\\
    -Obtain a set of centroids $C^i$ by applying $K$-means++ over the representatives of $\mathcal{P}$.\\
	- Compute $\epsilon_{S^i,C^i}(B)$  for all $B \in \mathcal{B}$ (Eq. \ref{eq:criterio}).
    
	}
\texttt{Step 4}: Compute $Pr(B)$ for every $B \in \mathcal{B}$, using $\epsilon_{S^i,C^i}(B)$ for $i=1,..,r$  (Eq. \ref{eq:cuttinProb}).\\
\Return $\Pr(\cdot)$.
\end{algorithm}

In Algorithm \ref{alg:computeProb}, we show the pseudo-code used in \texttt{Step 2} of Algorithm \ref{alg:InitialPartition} for computing the cutting probabilities associated to each block $B \in \mathcal{B}$, $Pr(B)$. Such a probability function depends on the misassignment function associated to each block with respect to multiple $K$-means++ based set of centroids. To generate these sets of centroids, $r$ subsamples of size $s$, with replacement, are extracted from the dataset, $D$. In particular, the \textbf{cutting probabilities} is expressed as follows:
\begin{equation}
\Pr(B)=  \frac{\sum_{i=1}^r \epsilon_{S^i,C^i}(B)}{\sum_{B' \in \mathcal{B}} \sum_{i=1}^r \epsilon_{S^i,C^i}(B')} \label{eq:cuttinProb}
\end{equation}
for each $B \in \mathcal{B}$, where $S^i$ is the subset of points sampled and $C^i$ is the set of $K$ centroids obtained via $K$-means++ for $i=1,...,r$. As we commented before, the larger the misassignment function is, then the more likely it is for the corresponding block to contain instances that belong to different clusters. It should be highlighted that a block $B$ with a cutting probability $Pr(B)=0$ is well assigned for all $S^i$ and $C^i$, with $i=1,..,r$.

Even when cheaper seeding procedures, such as a Forgy type initialization, could be used, $K$-means ++ avoids cluster oversampling, and so one would expect the corresponding boundaries not to divide subsets of points that are supposed to have the same cluster affiliation. Additionally, as previously commented, this initialization also tends to lead to competitive solutions. Later on, in Section \ref{sss:ParSe}, we will comment on the selection of the different parameters, used in the initialization ($m$, $m'$, $r$ and $s$).

\subsection{Construction of the sequence of thinner partitions} \label{SubSec:ThinnerP}

In this section, we provide the pseudo-code of the BW$K$M algorithm and introduce a detailed description of the construction of the sequence of thinner partitions, which is the basis of BW$K$M.
In general, once the initial partition is constructed via algorithm \ref{alg:InitialPartition}, BW$K$M progresses iteratively by alternating i) a run of weighted Lloyd's algorithm over the current partition and ii) the creation of a thinner partition using the information provided by the weighted Lloyd's algorithm. The pseudo-code of the BW$K$M algorithm can be seen in Algorithm \ref{alg:RPKM_A}.

\removelatexerror
\begin{algorithm}[H]\label{alg:RPKM_A}
    \caption{{\bf BW$K$M Algorithm}}

   {\bf Input:} Dataset $D$, number of clusters $K$ and initialization parameters $m'$, $m$, $s$, $r$.\\ 

   {\bf Output:} Set of centroids $C$.
   
   \vspace{0.10cm}
   
      \texttt{Step 1:} Initialize $\mathcal{B}$ and $\mathcal{P}$ via Algorithm \ref{alg:InitialPartition}, with input $m'$, $m$, $s$, $r$, and obtain $C$ by applying a weighted $K$-means++ run over the set of representatives of $\mathcal{P}$.\\
      
\texttt{Step 2:} $C=$ \texttt{WeightedLloyd}$(\mathcal{P}, C,K)$.\\      

   \While{not Stopping Criterion}{
   
  \texttt{Step 3:} Update dataset partition $\mathcal{P}$:\\

 - Compute $\epsilon_{C,D}(B)$  for all $B \in \mathcal{B}$.\\
 - Select $\mathcal{A}\subseteq \mathcal{F}_{C,D}(\mathcal{B})\subseteq \mathcal{B}$ by sampling, with replacement,  $|\mathcal{F}_{C,D}(\mathcal{B})|$ blocks according to $\epsilon_{C,D}(B)$, for all $B \in \mathcal{B}$.\\
 - Cut each block in $\mathcal{A}$ and update $\mathcal{B}$ and $\mathcal{P}$.

\texttt{Step 4:}  $C=$ \texttt{WeightedLloyd}$(\mathcal{P}, C,K)$.

}
\Return $C$
\end{algorithm}

In \texttt{Step 1}, the initial spatial partition $\mathcal{B}$ and the induced dataset partition, $\mathcal{P}=\mathcal{B}(D)$, are generated via Algorithm \ref{alg:InitialPartition}. Afterwards, the initial set of centroids is obtained through a weighted version of $K$-means++ over the set of representatives of $\mathcal{P}$.

Given the current set of centroids $C$ and the partition of the dataset $\mathcal{P}$, the set of centroids is updated in \texttt{Step 2} and \texttt{Step 4} by applying the weighted Lloyd's algorithm. It must be commented that the only difference between these two tasks is the fact that \texttt{Step 2} is initialized with a set of centroids obtained via weighted $K$-means++ run, while \texttt{Step 4} utilizes the set of centroids generated by the weighted Lloyd's algorithm over the previous dataset partition. In addition, in order to compute the misassignment function $\epsilon_{C,D}(B)$ for all $B \in \mathcal{B}$ in \texttt{Step 3} (see Eq.\ref{eq:criterio}), we store the following information provided by the last iteration of the weighted Lloyd's algorithm: for each $P \in \mathcal{P}$, the two closest centroids to the representative $\overline{P}$ in $C$ are saved (see Figure \ref{fig:TheoExp}).

In \texttt{Step 3}, a spatial partition thinner than $\mathcal{B}$ and its induced dataset partition are generated. For this purpose, the misassignment function, $\epsilon_{C,D}(B)$ for all $B \in \mathcal{B}$ is computed and the boundary $\mathcal{F}_{C,D}(\mathcal{B})$ is determined using the information stored at the last iteration of \texttt{Step 2}. Next, as the misassignment criterion in Theorem \ref{thm:FronteraGratis1} is just a sufficient condition, instead of dividing all the blocks that do not satisfy it, we prioritize those blocks that are less likely to be well assigned: A set $\mathcal{A}$ of  blocks is selected by sampling with replacement $|\mathcal{F}_{C,D}(\mathcal{B})|$ blocks from $\mathcal{B}$ with a (cutting) probability proportional to $\epsilon_{C,D}(B)$. Note that the size of $\mathcal{A}$ is at most $|\mathcal{F}_{C,D}(\mathcal{B})|$. Afterwards, in order to reduce as much as possible the length of the diagonal of the newly generated blocks and control the size of the thinner partition, each block in $\mathcal{A}$ is divided in the middle point of its largest side. Each block is split once into two equally shaped hyper-rectangles and it is replaced in $\mathcal{B}$ to produce the new thinner spatial partition. Finally, given the new spatial partition $\mathcal{B}$, its induced  dataset partition is obtained $\mathcal{P}=\mathcal{B}(D)$.

It should be noted that the cutting criterion, Eq.\ref{eq:criterio}, is more accurate, i.e., it detects more well assigned blocks, as long as we evaluate it over the smallest bounding box of each block of the spatial partition, since we minimize the maximum distance (diagonal) between any two points in the block. Therefore, when updating the data partition in \texttt{Step 3}, we also recompute the diagonal of the smallest bounding box of each subset.

\texttt{Step 2} and \texttt{Step 3} are then repeated until a certain stopping criterion is satisfied (for details on different stopping criteria, see Section \ref{sss:StCrit}).

\subsubsection{Computational complexity of the BW$K$M algorithm} \label{ss:Comp}

In this section, we provide the computational complexity of each step of BW$K$M, in the worst case.

The construction of the initial spatial partition, the corresponding induced dataset  partition and the set of centroids of BW$K$M (\texttt{Step 1}) has the following computational cost:
$\mathcal{O}(\max\{r \cdot s \cdot m^2, \ r \cdot K \cdot d \cdot m^2, \ \mathcal{O}( n\cdot\max\{ m, \ d\}) \})$. Each of the previous terms corresponds to the complexity of \texttt{Step 1}, \texttt{Step 2} and \texttt{Step 5} in Algorithm \ref{alg:InitialPartition}, respectively, which are the most computationally demanding procedures of the initialization. Even when these costs are deduced from the worst possible scenario, which is overwhelmingly improbable, in Section \ref{sss:ParSe}, we will comment on the selection of the initialization parameters in such a way that the cost of this step is not more expensive than that of the $K$-means algorithm, i.e., $\mathcal{O}(n \cdot K \cdot d)$.

As mentioned at the beginning of Section \ref{sec:MBK}, \texttt{Step 2} of Algorithm \ref{alg:RPKM_A}  (the weighted Lloyd's algorithm) has a computational complexity of $\mathcal{O}(|\mathcal{P}|\cdot K \cdot d)$. In addition, \texttt{Step 3} executes $\mathcal{O}(|\mathcal{P}| \cdot K)$ computations to verify the cutting criterion, since all the distance computations are obtained from the previous weighted Lloyd iteration. Moreover, assigning each instance to its corresponding block and updating the bounding box for each subset of the partition is $\mathcal{O}(n\cdot d)$. In summary, since $|\mathcal{P}|\leq n$, then BW$K$M algorithm has an overall computational complexity of $\mathcal{O}(n \cdot K \cdot d)$ in the worst case.

\subsection{Additional Remarks} \label{SubSec:AddCom}

In this section, we discuss additional features of the BW$K$M algorithm, such as the selection of the initialization parameters for BW$K$M, we also comment on different possible stopping criteria, with their corresponding computational costs and theoretical guarantees.

\subsubsection{Parameter selection}\label{sss:ParSe}

The construction of the initial space partition and the corresponding induced dataset partition of BW$K$M (see Algorithm \ref{alg:InitialPartition} and \texttt{Step 1} of Algorithm \ref{alg:RPKM_A}) depends on the parameters $m$, $m'$, $r$, $s$, $K$ and $D$, while the core of BW$K$M (\texttt{Step 2} and \texttt{Step 3}) only depends on $K$ and $D$. In this section, we propose how to select the parameters $m$, $m'$, $r$ and $s$, keeping in mind the following objectives: i) to guarantee BW$K$M  having a computational complexity equal to or lower than $\mathcal{O}(n\cdot K \cdot d)$, which corresponds to the cost of Lloyd's algorithm, and ii) to obtain an initial spatial partition with a large amount of well assigned blocks.

In order to ensure that the computational complexity of BW$K$M's initialization is, even in the worst case, $\mathcal{O}(n\cdot K \cdot d)$, we must take  $m$, $m'$, $r$ and $s$ such that $r \cdot s \cdot m^{2}$ , $r \cdot m^{2} \cdot K \cdot d$ and $n \cdot m$ are $\mathcal{O}(n \cdot K \cdot d)$. On the other hand, as we want such an initial partition to minimize the number of blocks that may not be well assigned, we must consider the following facts:  i) the larger the diagonal for a certain block $B \in \mathcal{B}$ is, then the more likely it is for $B$ not to be well assigned, ii) as the number of clusters $K$ increases, then any block $B \in \mathcal{B}$ has more chances of containing instances with different cluster affiliations, and iii) as $s$ increases, the cutting probabilities become better indicators for detecting those blocks that are not well assigned.

Taking into consideration these observations, and assuming that $r$ is a predefined small integer, satisfying $r \ll n/s$, we propose the use of $m=\mathcal{O}(\sqrt{K \cdot d})$ and $s=\mathcal{O}(\sqrt{n})$. Not only does such a choice satisfy the complexity constraints that we just mentioned (See Theorem \ref{thm:ComplexityIn} in Appendix \ref{App:AppendixA}), but also, in this case, the size of the initial partition increases with respect to both dimensionality of the problem and number of clusters: Since at each iteration, we divide a block only on one of its sides, then, as we increase the dimensionality, we need more cuts (number of blocks) to have a sufficient reduction of  its diagonal (observation i)). Analogously, the number of blocks and the size of the sampling increases with respect to the number of clusters and the actual size of the dataset, respectively (observation ii) and iii)). In particular, in the experimental section, Section \ref{Sec:Experimental}, we used $m=10\cdot \sqrt{K\cdot d}$, $s=\sqrt{n}$ and $r=5$.

\subsubsection{Stopping Criterion}\label{sss:StCrit}

As we commented in Section \ref{sec:CTB}, one of the advantages of constructing spatial partitions with only well assigned blocks is that our algorithm, under this setting, converges to a local minima of the $K$-means problem over the entire dataset and, therefore, there is no need to execute any further run of the BW$K$M algorithm as the set of centroids will remain the same for any thinner partition of the dataset:

\begin{restatable}{thm}{thirdtheo}\label{lemma:wellassigned2}
If $C$ is a fixed point of the weighted $K$-means algorithm for a spatial partition $\mathcal{B}$, for which all of its blocks are well assigned, then $C$ is a fixed point of the $K$-means algorithm on $D$. \footnote{The proof of Theorem \ref{lemma:wellassigned2} in Appendix \ref{App:AppendixA}.}
\end{restatable}

To verify this criterion, we can make use of the concept of boundary of a spatial partition (Definition \ref{def:boundary}). In particular, observe that if $\mathcal{F}_{C,D}(\mathcal{B})=\emptyset$, then one can guarantee that all the blocks of $\mathcal{B}$ are well assigned with respect to both $C$ and $D$. To check this, we just need to scan the misassignment function value for each block, i.e., it is just $\mathcal{O}(|\mathcal{P}|)$. In addition to this criterion, in this section we will propose three other stopping criteria:  

\begin{itemize}[leftmargin=*]
\item {\it A practical computational criterion}: We could set, in advance, the amount of computational resources that we are willing to use and stop when we exceed them. In particular, as the computation of distances is the most expensive step of the algorithm, we could set a maximum number of distances as a stopping criterion.
 \item {\it A Lloyd's algorithm type criterion}: As we mentioned in Section \ref{sec:KA}, the common practice is to run Lloyd's algorithm until the reduction of the error, after a certain iteration, is small, see Eq \ref{eq:condpar}. As in our weighted approximation we do not have access to the error $E^{D}(C)$, a similar approach is to stop the algorithm when the obtained set of centroids, in consecutive iterations, is smaller than a fixed threshold, ${\varepsilon}_{w}$. We can actually set this threshold in a way that the stopping criterion of Lloyd's algorithm is satisfied. For instance, for ${\varepsilon}_{w}=\sqrt{l^{2}+\frac{\varepsilon^2}{n^2}}-l$, if $\|C-C'\|_{\infty} \leq{\varepsilon}_{w}$, then the criterion in Eq.\ref{eq:condpar} is satisfied\footnote{See Theorem \ref{thm:epsilon} in Appendix \ref{App:AppendixA}}. However, this would imply additional $\mathcal{O}(K\cdot d)$ computations at each iteration.
\item {\it A criterion based on the accuracy of the weighted error}: We could also consider the bound obtained at Theorem \ref{thm:ErrBound2} and stop when it is lower than a predefined threshold. This will let us know how accurate our current weighted error is with respect to the error over the entire dataset. All the information in this bound is obtained from the weighted Lloyd iteration and the  information of the block and its computation is just $\mathcal{O}(|\mathcal{P}|)$.
\end{itemize}

\section{Experiments} \label{Sec:Experimental}

In this section, we perform a set of experiments so as to analyze the relation between the number of distances computed and the quality of the approximation for the {\bf BW$K$M} algorithm proposed in Section \ref{Sec:Contribution}. In particular, we compare the performance of BW$K$M with respect to different methods known for the quality of their approximations: Lloyd's algorithm initialized via i) Forgy ({\bf F$K$M}), ii) $K$-means++ ({\bf $K$M++}) and iii) the Markov chain Monte Carlo sampling based approximation of the $K$-means$++$ ({\bf $K$MC2}). From now on we will refer to these approaches as {\it Lloyd's algorithm based methods}. We also consider the Minibatch $K$-means, with batches $b=\{100,500,1000\}$ \footnote{Similar values were used in the original paper \cite{Sculley}.} ({\bf MB b}), which is particularly known for its efficiency due to the small amount of resources needed to generate its approximation. Additionally, we also present the results associated to the $K$-means$++$ initialization ({\bf $K$M++$\_$init}).

To have a better understanding of BW$K$M, we analyze its performance on a wide variety of well known real datasets (see Table \ref{table:data}) with different scenarios of the clustering problem: size of the dataset, $n$, dimension of the instances, $d$, and number of clusters, $K$.

\begin{table}[H]
\addtolength{\tabcolsep}{1.6pt}
 \begin{minipage}{1\linewidth}
  \begin{center}
\begin{tabular}{|c|r|r|c|}
\hline
{{\bf Dataset}} & {{\bf $n$}} &  {{\bf $d$}}\\ \hline
\it{Corel Image Features (CIF)}  & $68,037$ & $17$\\
\it{3D Road Network (3RN)}  & $434,874$ & $3$\\
\it{Gas Sensor (GS)}  & $4,208,259$ & $19$\\
\it{SUSY}  & $5,000,000$ & $19$\\
\it{Web Users Yahoo! (WUY)}  & $45,811,883$ & $5$\\
\hline
\end{tabular} \quad
\vspace{-0.15cm}
\caption{Information of the datasets.}\label{table:data}
\end{center}
\end{minipage}
\end{table}

The considered datasets have different features, ranging from small datasets with large dimensions ({\it CIF}) to large datasets with small dimensions ({\it WUY}). For each dataset, we have considered a different number of clusters, $K=\{3,9,27\}$. Given the random nature of the algorithms, each experiment has been repeated 40 times for each dataset and each $K$ value.

\begin{figure*}[b!]
\begin{center}
 \includegraphics[width=0.80\textwidth,height=0.28\textwidth]{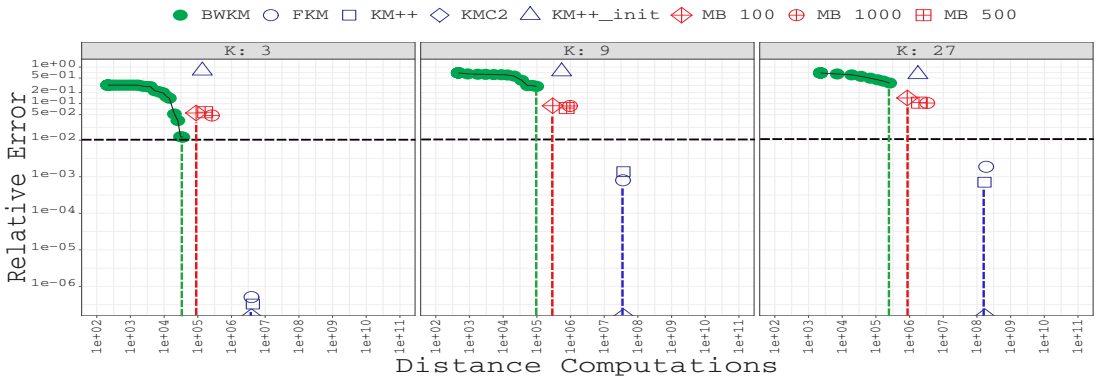}
    \caption{Distance computations {\it vs} relative error on the {\it CIF} dataset}
    \label{fig:ProBind}
\end{center}
\end{figure*}

\begin{figure*}
\begin{center}
  \includegraphics[width=0.80\textwidth,height=0.28\textwidth]{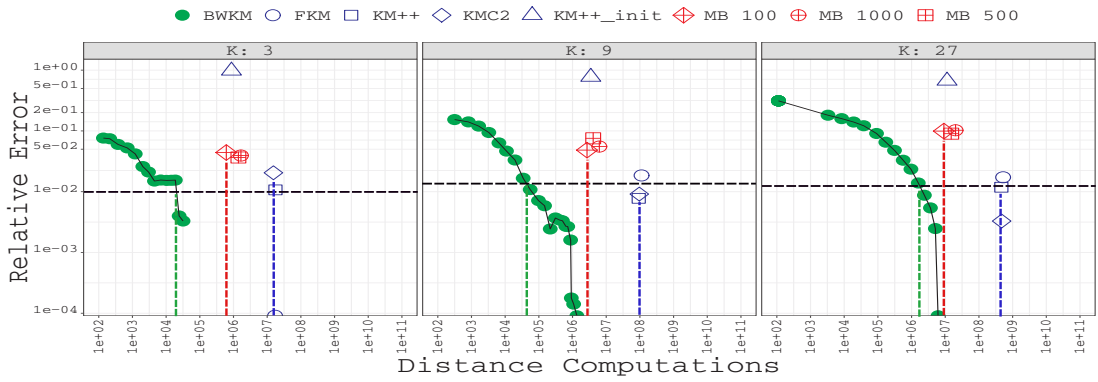}
    \caption{Distance computations {\it vs} relative error on the {\it 3RN} dataset}
    \label{fig:ProBind2}
\end{center}
\end{figure*}

In order to illustrate the competitiveness of our proposal, for each domain, we have limited its maximum number of distance computations to the minimum number required by the set of selected benchmark algorithms in all the 40 runs. Note that this distance bound is expected to be set by MB 100, since, among the methods that we consider, it is the one that uses the lowest number of representatives. 

However, BW$K$M can converge before reaching such a distance bound when the corresponding boundary is empty. In this case, we can guarantee that the obtained set of centroids is a fixed point of the weighted Lloyd's algorithm for any thinner partition of the dataset and, therefore, it is also a fixed point of Lloyd's algorithm on the entire dataset $D$ (see Theorem \ref{lemma:wellassigned2}).

The $K$-means error function (Eq.\ref{eq:errorfunction}) strongly depends on the different characteristics of the clustering problem: $n$, $K$, $d$ and the dataset itself. Thus, in order to compare the performance of the algorithms for different problems, we have decided to use the average of the relative error with respect to the best solution found at each repetition of the experiment:
\vspace{-0.10cm}
\begin{equation}
\hat{E}_{M}=  \frac{E_{M}-\min\limits_{M' \in \mathcal{M}} E_{M'}}{\min\limits_{M' \in \mathcal{M}} E_{M'}} \label{eq:RELERR}
\end{equation}
\vspace{-0.15cm}

{\setlength{\parindent}{0cm}
where $\mathcal{M}$ is the set of algorithms being compared and $E_{M}$ stands for the $K$-means error obtained by method $M \in \mathcal{M}$. That is, the quality of the approximation obtained by an algorithm $M \in \mathcal{M}$ is $100 \cdot \hat{E}_{\texttt{M}} \%$ worse than the best solution found by the set of algorithms considered.}

In Fig. \ref{fig:ProBind}-\ref{fig:ProBind5}, we show the trade-off between the average number of distances computed {\it vs} the average relative error for all the algorithms. Observe that a single symbol is used for each algorithm, except for BW$K$M, in which we compute the trade-off at each iteration so as to observe the evolution of the quality of its approximation as the number of computed distances increases.  Since the number of BW$K$M iterations required to reach the stopping criteria may differ at each execution, we plot the average of the most significant ones, i.e., those that do not exceed the upper limit of the  $95\%$ confidence interval of the total number of BW$K$M iterations for each run.
 
\setlength\belowcaptionskip{-3ex}

\begin{figure*}
\begin{center}
 
    \includegraphics[width=0.80\textwidth,height=0.28\textwidth]{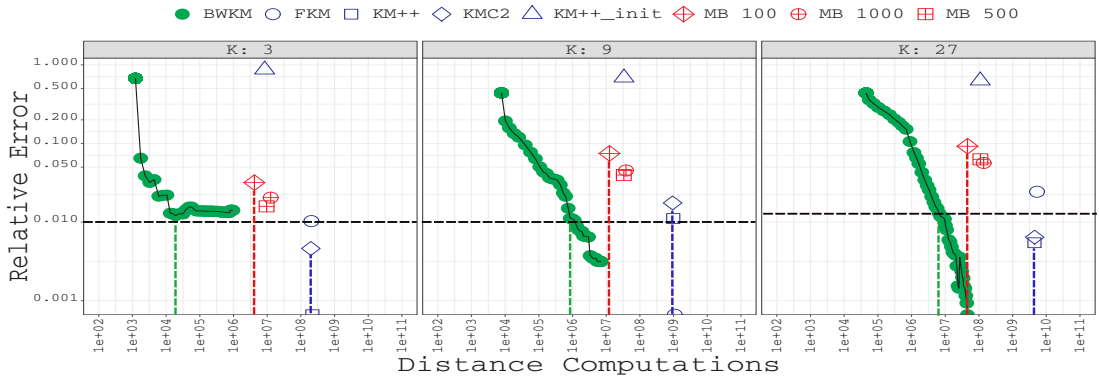}
    \caption{Distance computations {\it vs} relative error on the {\it GS} dataset}
    \label{fig:ProBind3}
\end{center}
\end{figure*}

\begin{figure*}
\begin{center}
    \includegraphics[width=0.80\textwidth,height=0.28\textwidth]{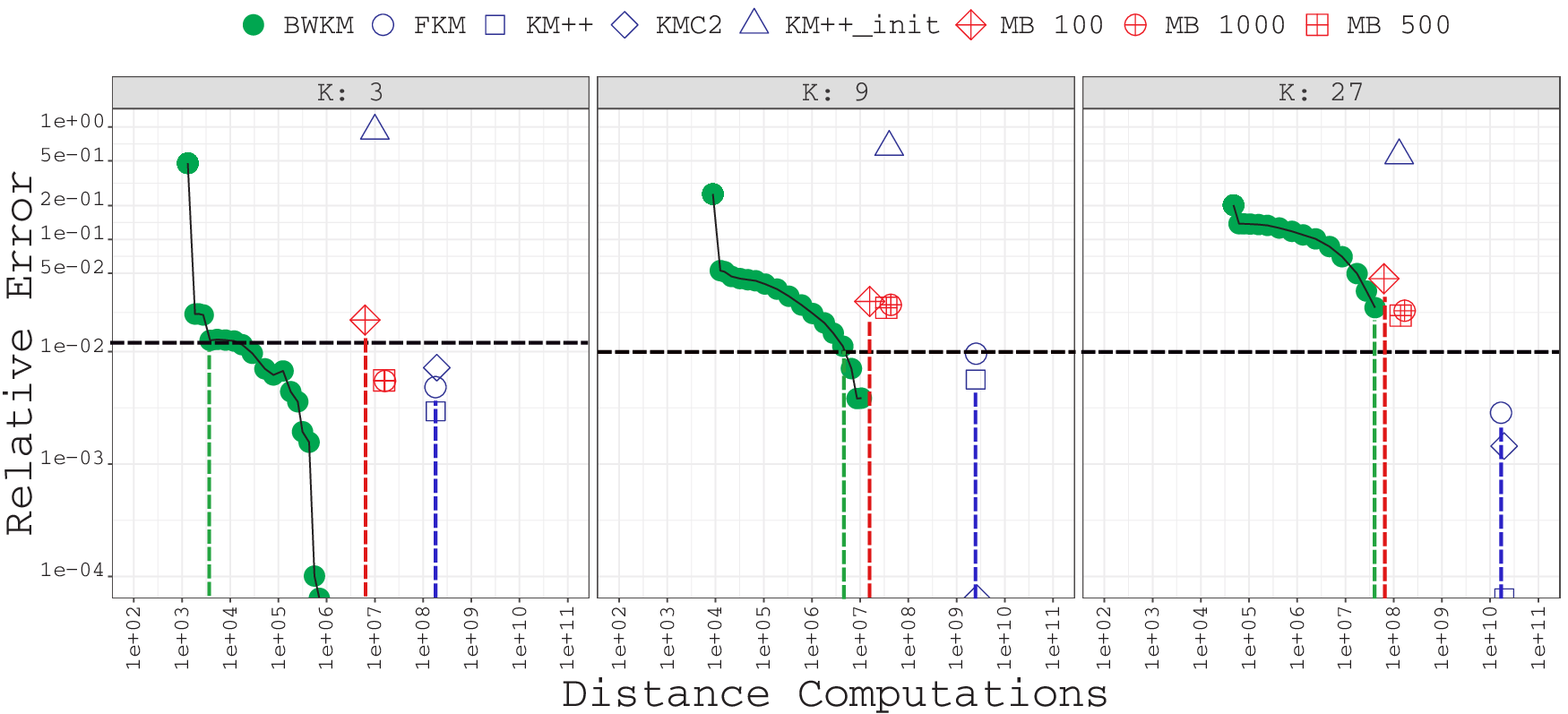}
    \caption{Distance computations {\it vs} relative error on the {\it SUSY} dataset}
    \label{fig:ProBind4}
\end{center}
\end{figure*}

\begin{figure*}
\begin{center}
  \includegraphics[width=0.80\textwidth,height=0.28\textwidth]{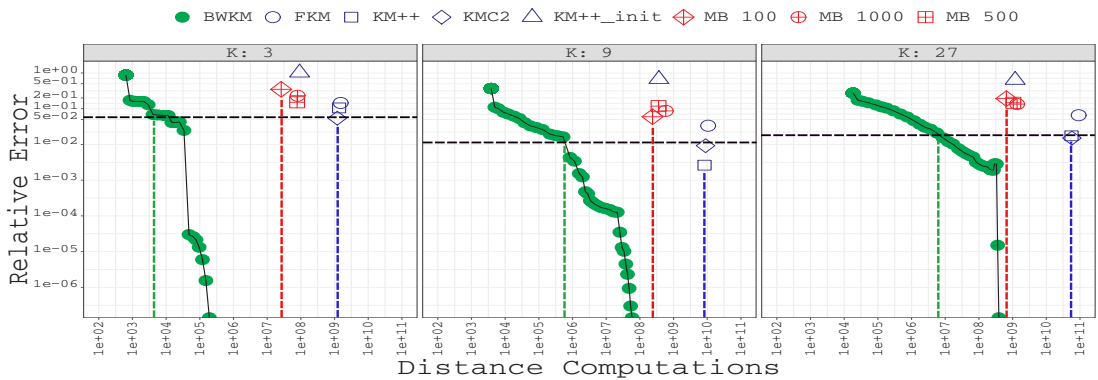}
    \caption{Distance computations {\it vs} relative error on the {\it WUY} dataset}
    \label{fig:ProBind5}
\end{center}
\end{figure*}

In order to ease the visualization of the results, both axis of each figure are in logarithmic scale. Moreover, we delimit with a horizontal dashed black line the regions of BW$K$M that are under $1\%$ of error with respect to the best found solution for the competition (Lloyd's algorithm based methods and MB) 
On one hand, the vertical green dashed line indicates the amount of distance computations required by BW$K$M to achieve such an error, when that happens, otherwise it shows the amount of distance computations at its last iteration. On the other hand, the blue and red vertical dashed lines show the algorithms among the Lloyd's algorithm based methods and MB that computed the least amount of distances, respectively. 

At first glance, we observe that, in $7$ out of $15$ different configurations of datasets and $K$ values, BW$K$M obtains the best (average) solution among the considered methods. It must be highlighted that such a clustering is achieved while computing a much reduced number of distances: up to 2 and 4 orders of magnitude of distances less than MB and the Lloyd's based methods, respectively. Moreover, BW$K$M quite frequently (in 12 out of 15 cases) generated solutions that reached, at least, $1\%$ of error with respect to the best solution found among the competitors (black dashed line). In particular and as expected, the best performance of BW$K$M seems to occur on large datasets with small dimensions ({\it WUY}). On one hand, the decrease in the amount of distances computed is mainly due to the reduction in the number of representatives that BW$K$M uses in comparison to the actual size of the dataset. On the other hand, given a set of points as the dimension decreases, the number of blocks required to obtain a partition completely well assigned tends to decrease ({\it WUY} and {\it 3RN}).

Regardless of this, even when considering the most unfavorable setting considered for BW$K$M (small dataset size and large dimensions, e.g., {\it CIF}), for small $K$ values, our proposal still managed to converge to competitive solutions at a fast rate. Note that for small $K$ values, since the number of centroids is small, one may not need to reduce the diagonal of the blocks so abruptly to verify the well assignment criterion.

Next, we will discuss in detail the results obtained for each of the considered databases. 

In the case of {\it CIF}, which is the smallest dataset and has a high dimensionality, BW$K$M behaves similarly to MB. It gets its best results for $K=3$, where it reaches $1\%$ of relative error with respect to the best solution found among the competitors, while reducing over 2 orders of magnitude of distances with respect to the Lloyd's based methods. For $K=\{9,27\}$, BW$K$M improves the results of $K$M++$\_$init using a much lower number of distances computed. 

In the case of small datasets with low dimensionality ({\it 3RN}), BW$K$M performs much better in comparison to the previous case: for $K \in \{9,27\}$, it actually generates the most competitive solutions. Moreover, in order to achieve a relative error of $1\%$ with respect to the best solution found by the benchmark algorithms, our algorithm reduces between 1 to 2 orders of magnitude of distances with respect to MB, and around 3 orders of magnitude against the Lloyd's based methods.

If we consider the case of the medium to large datasets with hight dimensionality ({\it GS} and {\it SUSY}), in order to reach a $1\%$ relative error, BW$K$M needs up to $3$ orders of magnitude less than MB and from $2$ to $5$ orders less than the Lloyd's based methods. Moreover, BW$K$M obtains the best results in $2$ out of $6$ configurations requiring $2$ order of magnitude less than the Lloyd's based algorithms. 

For the largest dataset with low dimension ({\it WUY}), BW$K$M got its best performance: Regardless of the number of clusters $K$, BW$K$M generated the most competitive solutions. Furthermore, in order to achieve a solution with an error $1\%$ higher than the best of the Lloyd's algorithm, BW$K$M requires  to compute an amount of distance from $2$ to $4$ and 4 to over 5 order of magnitude lower than MB and the Lloyd's based algorithms, respectively. 

Finally, we would like to highlight that BW$K$M, already at its first iterations, reaches a relative error much lower than $K$M++$\_$init in all the configurations requiring to compute an amount of distances from $3$ to $5$ order of magnitude lower. 
This fact strongly motivates the use of BW$K$M  as a competitive initialization strategy for Lloyd's algorithm. 

\section{Conclusions} \label{Sec:NextSteps}

In this work, we have presented an alternative to the $K$-means algorithm, oriented to massive data problems, called the Boundary Weighted $K$-means algorithm (BW$K$M). This approach recursively applies a weighted version of the $K$-means algorithm over a sequence of spatial based partitions of the dataset that ideally contains a large amount of {\it well assigned blocks}, i.e., cells of the spatial partition that only contain instances with the same cluster affiliation. It can be shown that our weighted error approximates the $K$-means error function, as we increase the number of well assigned blocks, see Theorem \ref{thm:ErrBound2}. Ultimately, if all the blocks of a spatial partition are well assigned at the end of a BW$K$M step, then the obtained clustering is actually a fixed point of the $K$-means algorithm, which is generated after using only a small number of representatives in comparison to the actual size of the dataset (Theorem \ref{lemma:wellassigned2}). Furthermore, if, for a certain step of BW$K$M, this property can be verified at consecutive weighted Lloyd's iterations, then the error of our approximation also decreases monotonically (Theorem \ref{lemma:wellassigned}).

In order to achieve this, in Section \ref{SubSec:DetectBound}, we designed a criterion to determine those blocks that may not be well assigned. One of the major advantages of the criterion is its low computational cost: It only uses information generated by the weighted $K$-means algorithm -distances between the center of mass of each block and the set of centroids- and a feature of the corresponding spatial partition -diagonal length of each block-. This allows us to guarantee that, even in the worst possible case, BW$K$M does not have a computational cost higher than that of the $K$-means algorithm. In particular, the criterion is presented in Theorem \ref{thm:FronteraGratis1} and states that, if the diagonal of a certain block is smaller than half the difference of the two the smallest distances between its center of mass and the set of centroids, then the block is well assigned. 

In addition to all the theoretical guarantees that motivated and justify our algorithm (see Section \ref{Sec:Contribution} and Appendix \ref{App:AppendixA}), in practice, we have also observed its competitiveness with respect to the state-of-the-art (Section \ref{Sec:Experimental}). BW$K$M has been compared to techniques known for the quality of their approximation (Lloyd's algorithm initialized with Forgy's approach, $K$-means++ and via an approximation of the $K$-means++ probability function based on a Markov chain Monte Carlo sampling). Besides, it has been compared to Minibatch $K$-means, a method known for the small amount of computational resources that it needs for approximating the solution of the $K$-means problem. 

The results, on different well known real datasets, show that BW$K$M in several cases (7 out of 15 configurations) has generated the most competitive solutions. Furthermore, in 12 out of 15 cases, BW$K$M has converged to solutions with a relative error of under $1\%$ with respect to the best solution found by the state-of-the-art, while using a much smaller amount of distance computations (from 2 to 6 orders of magnitude lower).

As for the next steps, we plan to exploit different benefits of BW$K$M. First of all, observe that the proposed algorithm is embarrassingly parallel up to the $K$-means++ seeding of the initial partition (over a very tiny amount of representatives when compared to the dataset size), hence we could implement this approach in a more appropriate platform for this kind of problems, as is the case of {\it Apache Spark}. Moreover, we must point out that BW$K$M is also compatible with the distance pruning techniques presented in \cite{Drake,Elkan,Hamerly}, therefore, we could also implement these techniques within the weighted Lloyd framework of BW$K$M and reduce, even more, the number of distance computations. 

\appendix \label{App:AppendixA}

In the first result, we present a complimentary property of the grid based RP$K$M proposed in \cite{Capo}. Each iteration of the RP$K$M can be proved to be a coreset with an exponential decrease in the error with respect to the number of iterations. This result could actually bound the BW$K$M error, if we fix $i$ as the minimum number of cuts that a block, of a certain partition generated by BW$K$M, $\mathcal{P}$, has.

\begin{thm}\label{theo:coreset}
Given a set of points $D$ in $\mathbb{R}^d$, the $i$-th iteration of the grid 
based RP$K$M produces a $(K,\varepsilon)$-coreset with 
$\varepsilon=\frac{1}{2^{i-1}}\cdot(1+\frac{1}{2^{i+2}} \cdot \frac{n-1}{n})\cdot \frac{n \cdot l^2}{OPT}$,
where $OPT=\min\limits_{C \subseteq \mathbb{R}^d, |C|=K} E^{D}(C)$ and $l$ the length 
of the diagonal of the smallest bounding box containing $D$.
\end{thm}

\begin{proof}
Firstly, we denote by $\textbf{x}'$ to the representative of $\textbf{x}\in D$ at the $i$-th grid based RP$K$M iteration, i.e., 
if $\textbf{x}\in P$ then $\textbf{x}'=\overline{P}$, where $P$ is a block of the corresponding dataset partition $\mathcal{P}$ of $D$. Observe that, at the $i$-th grid based RP$K$M iteration, the length of the diagonal of each cell is $\frac{1}{2^i}\cdot l$ and we set a positive constant, $c$, as the positive real number satisfying $\frac{1}{2^i}\cdot l=\sqrt{c \cdot \frac{OPT}{n}}$. By the triangular inequality, we have

 \begin{align}
|E^{D}&(C)-E^{\mathcal{P}}(C)|\leq \sum\limits_{\textbf{x}\in D} |\| \textbf{x}-\textbf{c}_{\textbf{x}}\|^2-\| \textbf{x}'-\textbf{c}_{\textbf{x}'}\|^2|\nonumber\\
&\leq \sum\limits_{\textbf{x}\in D} |(\| \textbf{x}-\textbf{c}_{\textbf{x}}\|-\| \textbf{x}'-\textbf{c}_{\textbf{x}'}\|)(\| \textbf{x}-\textbf{c}_{\textbf{x}}\|+\| \textbf{x}'-\textbf{c}_{\textbf{x}'}\|)|\nonumber
\end{align}

 Analogously, observe that the following inequalities hold $\| \textbf{x}'-\textbf{c}_{\textbf{x}'}\|+\|\textbf{x}-\textbf{x}'\|\geq \| \textbf{x}-\textbf{c}_{\textbf{x}}\|$
and $\| \textbf{x}-\textbf{c}_{\textbf{x}}\|+\|\textbf{x}-\textbf{x}'\|\geq \| \textbf{x}'-\textbf{c}_{\textbf{x}'}\|$. Thus, $\|\textbf{x}-\textbf{x}'\|\geq |\| \textbf{x}-\textbf{c}_{\textbf{x}}\|-\| \textbf{x}'-\textbf{c}_{\textbf{x}'}\||$:

 \begin{eqnarray}
|E^{D}(C)-E^{\mathcal{P}}(C)|\leq \sum\limits_{\textbf{x}\in D} \|\textbf{x}-\textbf{x}'\| \cdot (2\cdot \| \textbf{x}-\textbf{c}_{\textbf{x}}\|+\|\textbf{x}-\textbf{x}'\|)\nonumber
\end{eqnarray}

On the other hand, we know that $\sum\limits_{\textbf{x}\in D} \|\textbf{x}-\textbf{x}'\|^2 \leq \frac{n-1}{2^{2i+1}}\cdot l^{2}$ and that, as both $\textbf{x}$ and $\textbf{x}'$ must be located in the same cell, $\|\textbf{x}-\textbf{x}'\|\leq \frac{1}{2^i}\cdot l$. Therefore, as ${\bf d}(\textbf{x},C) \leq l$,

 \begin{eqnarray}
|E^{D}(C)-E^{\mathcal{P}}(C)|&\leq& (\frac{n-1}{2^{2i+1}}+\frac{n}{2^{i-1}})\cdot l^{2} \nonumber\\
&\leq& (\frac{n-1}{2^{2i+1}}+\frac{n}{2^{i-1}}) \cdot 2^{2i} \cdot c \cdot \frac{OPT}{n}\nonumber\\
&\leq& (\frac{1}{2^{i+2}} \cdot \frac{n-1}{n}+1) \cdot 2^{i+1} \cdot c \cdot E(C) \nonumber
\end{eqnarray}

In other words, the $i$-th RP$K$M iteration is a $(K,\varepsilon)$- coreset with 
$\varepsilon=(\frac{1}{2^{i+2}} \cdot \frac{n-1}{n}+1) \cdot 2^{i+1} \cdot c=\frac{1}{2^{i-1}}\cdot(1+\frac{1}{2^{i+2}} \cdot \frac{n-1}{n})\cdot \frac{n \cdot l^2}{OPT}$.
 \end{proof}
 
 The two following results show some properties of the error function when having well assigned blocks. 
 
\begin{lemma}\label{theo:wellassigned}
If $\textbf{c}_{\textbf{x}}=\textbf{c}_{\overline{P}}$ and $\textbf{c}_{\textbf{x}}'=\textbf{c}_{\overline{P}}'$ for all $\textbf{x} \in P$, where $P \subseteq D$ and $C$, $C'$ are a  pair of sets of centroids, then $E^{P}(C)-E^{\{ P \}}(C)=E^{P}(C')-E^{\{ P \}}(C')$.
\end{lemma}

\begin{proof}
From Lemma 1 in \cite{Capo}, we can say that the following function is constant
$f(\textbf{c})= |P|\cdot\| \overline{P}- \textbf{c} \|^2 - \sum_{\textbf{x} \in P} \| \textbf{x}- \textbf{c} \|^2 $, for $\textbf{c} \in \mathbb{R}^d$.
In particular, since $f(\overline{P})=- \sum_{\textbf{x} \in P} \| \textbf{x}- \overline{P} \|^2$, we have that
$|P|\cdot\| \overline{P}- \textbf{c}_{\overline{P}} \|^2  =\sum_{\textbf{x} \in P} \| \textbf{x}- \textbf{c}_{\overline{P}} \|^2 - \sum_{\textbf{x} \in P} \| \textbf{x}- \overline{P} \|^2$
and so we can express the weighted error of a dataset partition, $\mathcal{P}$, as follows

\begin{eqnarray}\label{EQ2}
E^{\mathcal{P}}(C)= \sum\limits_{P \in \mathcal{P}} \sum_{\textbf{x} \in P} \| \textbf{x}- \textbf{c}_{\overline{P}} \|^2 - \| \textbf{x}- \overline{P} \|^2
\end{eqnarray}

In particular, for $P \in \mathcal{P}$, we have 

\begin{align}\label{EQ33}
 E^{P}(C)-E^{\{ P \}}(C)&=\sum_{\textbf{x} \in P} \| \textbf{x}- \textbf{c}_{\textbf{x}} \|^2 -\| \textbf{x}- \textbf{c}_{\overline{P}} \|^2 + \| \textbf{x}- \overline{P} \|^2 \nonumber\\
 &= \sum_{\textbf{x} \in P}  \| \textbf{x}- \overline{P} \|^2 \nonumber\\
 &=\sum_{\textbf{x} \in P} \| \textbf{x}- \textbf{c}_{\textbf{x}}' \|^2 -\| \textbf{x}- \textbf{c}_{\overline{P}}'  \|^2 + \| \textbf{x}- \overline{P} \|^2 \nonumber\\
 &= E^{P}(C')-E^{\{ P \}}(C') \nonumber
\end{align}
\end{proof}
 
 In the previous result we observe that, if all the instances are correctly assigned in each block, then the difference of the weighted and the  entire dataset error, of both sets of centroids, is the same. In other words, if all the blocks of a given partition are correctly assigned, not only can we then actually guarantee a monotone descend of the entire error function for our approximation, a property that can not be guaranteed for the typical coreset type approximations of $K$-means, but we know exactly the reduction of such an error after a weighted Lloyd iteration.

\begin{thm}\label{lemma:wellassigned}

Given two set of centroids $C$, $C'$, where $C'$ is obtained after a weighted Lloyd's
iteration (on a partition $\mathcal{P}$) over $C$ and $\textbf{c}_{\textbf{x}}=\textbf{c}_{\overline{P}}$ and
$\textbf{c}_{\textbf{x}}'=\textbf{c}_{\overline{P}}'$ for all $\textbf{x} \in P$ and $P \in \mathcal{P}$, then
$E^{D}(C') \leq E^{D}(C)$.
\end{thm}

\begin{proof}
Using Lemma \ref{theo:wellassigned} over all the subsets $P \in \mathcal{P}$, we know that
$E^{D}(C')-E^{D}(C)=\sum_{P \in \mathcal{P}} (E^{P}(C')-E^{P}(C))$ \
$= \sum_{P \in \mathcal{P}} (E^{\{ P \}}(C')-E^{\{ P \}}(C))=E^{\mathcal{P}}(C')-E^{\mathcal{P}}(C)$.
Moreover, from the chain of inequalities $A.1$ in \cite{Capo}, we know that $E^{\mathcal{P}}(C')\leq E^{\mathcal{P}}(C)$
at any weighted Lloyd iteration over a given partition $\mathcal{P}$, thus 
$E^{D}(C')\leq E^{D}(C)$.
 \end{proof}

 In Theorem \ref{thm:FronteraGratis1}, we prove the cutting criterion that we use in BW$K$M. It consists of an inequality that, only by using information referred to the partition of the dataset and the weighted Lloyd's algorithm, helps us guarantee that a block is well assigned.
 
{\renewcommand\footnote[1]{}\primetheo*}

\begin{proof}
 From the triangular inequality, we know that $\| \textbf{x}-\textbf{c}_{\overline{P}}\| \leq \| \textbf{x}-\overline{P}\|+\| \overline{P}-\textbf{c}_{\overline{P}}\|$. Moreover, observe that $\overline{P}$ is contained in the block $B$, since $B$ is a convex polygon. Then $\| \textbf{x}-\overline{P}\| < l_{B}$.
 
For this reason, $\| \textbf{x}-\textbf{c}_{\overline{P}}\| < l_{B} - \delta_{P}(C) + \| \overline{P}-\textbf{c}\|< (2\cdot l_{B} - \delta_{P}(C)) + \| \textbf{x}-\textbf{c}\|$ holds.
As $\epsilon_{C,D}(B)=\max\{0,2\cdot l_{B}-\delta_{P}(C)\}=0$, then
$2\cdot l_{B}-\delta_{P}(C) \leq 0$ and, therefore,
 $\| \textbf{x}-\textbf{c}_{\overline{P}}\| < \| \textbf{x}-\textbf{c}\|$ for all $\textbf{c} \in C$. In other words,
 $\textbf{c}_{\overline{P}}=\argmin\limits_{\textbf{c} \in C} \| \textbf{x}-\textbf{c}\|$ for all $\textbf{x} \in P$.
\end{proof}

As can be seen in Section \ref{SubSec:InitialPartition}, there are different parameters that must be tuned. In the following result, we set a criterion to choose the initialization parameters of Algorithm \ref{alg:InitialPartition} in a way that its complexity, even in the worst case scenario, is still the same as that of Lloyd's algorithm. 

\begin{thm}\label{thm:ComplexityIn}
Given an integer $r$, if $m=\mathcal{O}(\sqrt{K \cdot d})$ 
and $s=\mathcal{O}(\sqrt{n})$,
then Algorithm \ref{alg:InitialPartition} is $\mathcal{O}(n\cdot K \cdot d)$.
\end{thm}

\begin{proof}
It is enough to verify the conditions presented before. Firstly, observe that 
$r \cdot s \cdot m^2=\mathcal{O}(\sqrt{n} \cdot K \cdot d)$ and
$n \cdot m=\mathcal{O}(n \cdot \sqrt{K \cdot d})$. Moreover, 
as $K\cdot d= \mathcal{O}(n)$, then
$r \cdot m^2=\mathcal{O}(n)$.
\end{proof}

Up to this point, most of the quality results assume the case when all the blocks are well assigned. However, in order to achieve this, many BW$K$M iterations might be required. In the following result, we provide a bound to the weighted error with respect to the full error. This result shows that our weighted representation improves as more blocks of our partition satisfy the criterion in Algorithm \ref{thm:FronteraGratis1} and/or
the diagonal of the blocks are smaller.

{\renewcommand\footnote[1]{}\sectheo*}
\begin{proof}
Using Eq.\ref{EQ2} in Theorem \ref{theo:wellassigned}, we know that
$|E^{D}(C)-E^{\mathcal{P}}(C)| \leq \sum\limits_{P \in \mathcal{P}} \sum\limits_{\textbf{x} \in P} \| \textbf{x}- \textbf{c}_{\overline{P}} \|^2 -\| \textbf{x}- \textbf{c}_{\textbf{x}} \|^2 + \| \textbf{x}- \overline{P} \|^2$.

Observe that, for a certain instance $\textbf{x} \in P$, where 
$\epsilon_{C,D}(B)=\max\{0,2\cdot l_{B}-\delta_{P}(C)\}=0$,
$\| \textbf{x}- \textbf{c}_{\overline{P}} \|^2 -\| \textbf{x}- \textbf{c}_{\textbf{x}} \|^2=0$, as 
$\textbf{c}_{\textbf{x}}=\textbf{c}_{\overline{P}}$ by Theorem \ref{thm:FronteraGratis1}. On the other hand,
if $\epsilon_{C,D}(B) > 0$,
we have the following inequalities:

 \begin{align}
 \| \textbf{x}- \textbf{c}_{\overline{P}} \| -\| \textbf{x}- \textbf{c}_{\textbf{x}} \|&\leq 2\cdot \| \textbf{x}- \overline{P} \| -(\| \overline{P}- \textbf{c}_{\textbf{x}} \|-\| \overline{P}- \textbf{c}_{\overline{P}} \|) \nonumber\\
 &\leq \epsilon_{C,D}(B) \nonumber
\end{align}

\begin{align}
 \| \textbf{x}- \textbf{c}_{\overline{P}} \| +\| \textbf{x}- \textbf{c}_{\textbf{x}} \|&\leq 2\cdot \| \textbf{x}- \overline{P} \| + \| \overline{P}- \textbf{c}_{\textbf{x}} \|+\| \overline{P}- \textbf{c}_{\overline{P}} \| \nonumber\\
 &< 2\cdot l_{B}+(2\cdot l_{B}+\| \overline{P}- \textbf{c}_{\overline{P}} \|) \nonumber \\
 &+ \| \overline{P}- \textbf{c}_{\overline{P}} \| \nonumber \\
 &= 2\cdot (2\cdot l_{B}+\| \overline{P}- \textbf{c}_{\overline{P}} \|) \nonumber
\end{align}

Using both inequalities, we have 
$\| \textbf{x}- \textbf{c}_{\overline{P}} \|^2 -\| \textbf{x}- \textbf{c}_{\textbf{x}} \|^2 \leq 2 \cdot \epsilon_{C,D}(B) \cdot (2\cdot l_{B}+\| \overline{P}- \textbf{c}_{\overline{P}} \|)$. On the other hand, observe that
$\sum\limits_{\textbf{x} \in P} \| \textbf{x}- \overline{P} \|^2= \frac{1}{|P|} \cdot \sum\limits_{\textbf{x},\textbf{y} \in P} \| \textbf{x}- \textbf{y} \|^2 \leq \frac{1}{|P|} \cdot\frac{|P|\cdot(|P|-1)}{2} \cdot l_{B}^2=\frac{|P|-1}{2} \cdot l_{B}^2$.
 \end{proof}
 
 As we do not have access to the error for the entire dataset, $E^{D}(C)$, since its computation is expensive, in Algorithm \ref{alg:RPKM_A}
 we propose a possible stopping criterion that bounds the displacement of the set of centroids. In the following result, we show a possible 
 choice of this bound in a way that, if the proposed criterion is verified, then the common Lloyd's algorithm stopping criterion
 is also satisfied.
 
 \begin{thm}\label{thm:epsilon}
Given two sets of centroids $C=\{\textbf{c}_{k}\}_{k=1}^{K}$ and $C'=\{\textbf{c}_{k}'\}_{k=1}^{K}$ , if
$\|C-C'\|_{\infty}=\max\limits_{k=1, \ldots,K} \|\textbf{c}_{k}-\textbf{c}_{k}' \|\leq {\varepsilon}_{w}$, where 
${\epsilon}_{w}=\sqrt{l^{2}+\frac{\epsilon^2}{n^2}}-l$, then 
$|E^{D}(C)-E^{D}(C')|\leq \varepsilon$.
\end{thm}

\begin{proof}
Initially, we bound the following terms: $\| \textbf{x}-\textbf{c}_{\textbf{x}}\|+\| \textbf{x}-\textbf{c}_{\textbf{x}}'\|$
and $|\| \textbf{x}-\textbf{c}_{\textbf{x}}\|-\| \textbf{x}-\textbf{c}_{\textbf{x}}'\||$ for any $\textbf{x}\in D$.

If we set $j$ and $t$ as the indexes satisfying
$\textbf{c}_{j}=\textbf{c}_{\textbf{x}}$ and
$\textbf{c}_{t}'=\textbf{c}_{\textbf{x}}'$, then we have
$\| \textbf{x}-\textbf{c}_{\textbf{x}}\|+\| \textbf{x}-\textbf{c}_{\textbf{x}}'\|=\| \textbf{x}-\textbf{c}_{j}\|+\| \textbf{x}-\textbf{c}_{t}'\|\leq \| \textbf{x}-\textbf{c}_{t}\|+\| \textbf{x}-\textbf{c}_{t}'\|\leq 2 \cdot \| \textbf{x}-\textbf{c}_{t}'\|+\varepsilon_{w}= 2 \cdot \| \textbf{x}-\textbf{c}_{\textbf{x}}'\|+\varepsilon_{w}$ (1).
Analogously, applying the triangular inequality, we have
$|\| \textbf{x}-\textbf{c}_{\textbf{x}}\|-\| \textbf{x}-\textbf{c}_{\textbf{x}}'\||\leq \varepsilon_{w}$ (2).
In the following chain of inequalities, we will make use of (1) and (2):

\begin{eqnarray}\label{ChainEps}
|E^{D}(C)-E^{D}(C')|&\leq& |\sum\limits_{\textbf{x}\in D} \| \textbf{x}-\textbf{c}_{\textbf{x}}\|^2-\| \textbf{x}-\textbf{c}_{\textbf{x}}'\|^2|\nonumber\\
 &\leq& \sum\limits_{\textbf{x}\in D} |\| \textbf{x}-\textbf{c}_{\textbf{x}}\|^2-\| \textbf{x}-\textbf{c}_{\textbf{x}}'\|^2|\nonumber \\
 &\leq& \sum\limits_{\textbf{x}\in D} (\| \textbf{x}-\textbf{c}_{\textbf{x}}\|+\| \textbf{x}-\textbf{c}_{\textbf{x}}'\|)\cdot \nonumber \\
 && |\| \textbf{x}-\textbf{c}_{\textbf{x}}\|-\| \textbf{x}-\textbf{c}_{\textbf{x}}'\||\nonumber \\
 &\leq& \sum\limits_{\textbf{x}\in D} \varepsilon_{w} \cdot (2 \cdot \|\textbf{x}-\textbf{c}_{\textbf{x}}'\|+\varepsilon_{w}) \nonumber \\
  &\leq& n \cdot \varepsilon_{w}^{2} + 2 \cdot n \cdot \max\limits_{\textbf{x} \in D}\|\textbf{x}-\textbf{c}_{\textbf{x}}'\|\cdot \varepsilon_{w} \nonumber \\
    &\leq& n \cdot \varepsilon_{w}^{2} + 2 \cdot n \cdot l\cdot \varepsilon_{w}= \varepsilon \nonumber
 \end{eqnarray}

 \end{proof}

 In Theorem \ref{lemma:wellassigned2}, we show an interesting property of the BW$K$M algorithm. We verify that a fixed point of the weighted Lloyd's algorithm, over a partition with only well assigned blocks, is also a fixed point of Lloyd's algorithm over the entire dataset $D$. 

{\renewcommand\footnote[1]{}\thirdtheo*}
\begin{proof}

$C=\{\textbf{c}_1, \ldots, \textbf{c}_K\}$ is a fixed point of the weighted $K$-means algorithm, on a partition $\mathcal{P}$, if and only if
when applying an additional iteration of the weighted $K$-means algorithm on $\mathcal{P}$,
the generated clusterings $\mathcal{G}_{1}(\mathcal{P}), \ldots, \mathcal{G}_{K}(\mathcal{P})$, i.e., 
$\mathcal{G}_{i}(\mathcal{P}) \coloneqq \{P \in \mathcal{P}: \textbf{c}_{i}=\argmin\limits_{\textbf{c} \in C} \|\overline{P}-\textbf{c}\| \}$, satisfies
$\textbf{c}_{i}=\frac{\sum\limits_{P \in \mathcal{G}_{i}(\mathcal{P})} |P| \cdot \overline{P}}{\sum\limits_{P \in \mathcal{G}_{i}(\mathcal{P})} |P|}$ for all $i=\{1, \ldots, K\}$  (1).

Since all the blocks $B \in \mathcal{B}$ are well assigned, then the clusterings of $C$ in $D$,
$\mathcal{G}_{i}(D) \coloneqq \{\textbf{x} \in D: \textbf{c}_{i}=\argmin\limits_{\textbf{c} \in C} \|\textbf{x}-\textbf{c}\| \}$,
satisfy $|\mathcal{G}_{i}(D)| = \sum\limits_{P \in \mathcal{G}_{i}(\mathcal{P})} |P|$ (2) and
$\sum\limits_{\textbf{x} \in  \mathcal{G}_{i}(D)} \textbf{x}=\sum\limits_{P \in \mathcal{G}_{i}(\mathcal{P})} \sum\limits_{\textbf{x} \in P} \textbf{x}$ (3). From (1), (2) and (3), we have

\begin{eqnarray}\label{EQ3}
\textbf{c}_{i}&=&\frac{\sum\limits_{P \in \mathcal{G}_{i}(\mathcal{P})} |P| \cdot \overline{P}}{\sum\limits_{P \in \mathcal{G}_{i}(\mathcal{P})} |P|}= \frac{\sum\limits_{P \in \mathcal{G}_{i}(\mathcal{P})} |P| \cdot \sum\limits_{\textbf{x} \in P} \frac{\textbf{x}}{|P|}}{\sum\limits_{P \in \mathcal{G}_{i}(\mathcal{P})} |P|}\nonumber\\
&=& \frac{\sum\limits_{P \in \mathcal{G}_{i}(\mathcal{P})} \sum\limits_{\textbf{x} \in P} \textbf{x}}{\sum\limits_{P \in \mathcal{G}_{i}(\mathcal{P})} |P|}=\frac{\sum\limits_{\textbf{x} \in  \mathcal{G}_{i}(D)} \textbf{x}}{|\mathcal{G}_{i}(D)|} \forall \ i \in{1,\ldots, K}, \nonumber
\end{eqnarray}
this is, $C$ is a fixed point of $K$-means algorithm on $D$.
 \end{proof}

\ifCLASSOPTIONcompsoc
  \section*{Acknowledgments}
\else
  \section*{Acknowledgment}
\fi

Marco Capó and Aritz Pérez are partially supported by the Basque Government through the BERC 2014-2017 program and the ELKARTEK program, and by the Spanish Ministry of Economy and Competitiveness MINECO: BCAM Severo Ochoa excelence accreditation SVP-2014-068574 and SEV-2013-0323. Jose A. Lozano is partially supported by the Basque Government (IT609-13), and Spanish Ministry of Economy and Competitiveness MINECO (BCAM Severo Ochoa excellence accreditation SEV-2013-0323 and TIN2016-78365-R).

\ifCLASSOPTIONcaptionsoff
  \newpage
\fi

\end{document}